\documentclass[11pt]{article}

\usepackage{stmaryrd}

\usepackage{amssymb}
\usepackage{chicagor}
\usepackage{times}

\newtheorem{THEOREM}{Theorem}[section]
\newenvironment{theorem}{\begin{THEOREM} \hspace{-.85em} {\bf :} }%
                        {\end{THEOREM}}
\newtheorem{LEMMA}[THEOREM]{Lemma}
\newenvironment{lemma}{\begin{LEMMA} \hspace{-.85em} {\bf :} }%
                      {\end{LEMMA}}
\newtheorem{COROLLARY}[THEOREM]{Corollary}
\newenvironment{corollary}{\begin{COROLLARY} \hspace{-.85em} {\bf :} }%
                          {\end{COROLLARY}}
\newtheorem{PROPOSITION}[THEOREM]{Proposition}
\newenvironment{proposition}{\begin{PROPOSITION} \hspace{-.85em} {\bf :} }%
                            {\end{PROPOSITION}}
\newtheorem{DEFINITION}[THEOREM]{Definition}
\newenvironment{definition}{\begin{DEFINITION} \hspace{-.85em} {\bf :} \rm}%
                            {\end{DEFINITION}}
\newtheorem{CLAIM}[THEOREM]{Claim}
\newenvironment{claim}{\begin{CLAIM} \hspace{-.85em} {\bf :} \rm}%
                            {\end{CLAIM}}
\newtheorem{EXAMPLE}[THEOREM]{Example}
\newenvironment{example}{\begin{EXAMPLE} \hspace{-.85em} {\bf :} \rm}%
                            {\end{EXAMPLE}}
\newtheorem{REMARK}[THEOREM]{Remark}
\newenvironment{remark}{\begin{REMARK} \hspace{-.85em} {\bf :} \rm}%
                            {\end{REMARK}}

\newcommand{\thm}{\begin{theorem}}
\newcommand{\lem}{\begin{lemma}}
\newcommand{\pro}{\begin{proposition}}
\newcommand{\dfn}{\begin{definition}}
\newcommand{\rem}{\begin{remark}}
\newcommand{\xam}{\begin{example}}
\newcommand{\cor}{\begin{corollary}}

\newcommand{\ethm}{\end{theorem}}
\newcommand{\elem}{\end{lemma}}
\newcommand{\epro}{\end{proposition}}
\newcommand{\edfn}{\bbox\end{definition}}
\newcommand{\erem}{\bbox\end{remark}}
\newcommand{\exam}{\bbox\end{example}}
\newcommand{\ecor}{\end{corollary}}
\newcommand{\eprf}{\bbox\vspace{0.1in}}
\newcommand{\beqn}{\begin{equation}}
\newcommand{\eeqn}{\end{equation}}

\newcommand{\bbox}{\vrule height7pt width4pt depth1pt}
\newcommand{\qed}{\eprf}
\newcommand{\clm}{\begin{claim}}
\newcommand{\eclm}{\end{claim}}

\newcommand{\boldsymbol}[1]{\mbox{\boldmath $\bf #1$}}

\DeclareMathAlphabet{\mathitbf}{OML}{cmm}{b}{it}

\newcommand{\eol}{\end{enumerate}\setlength{\itemsep}{-\parsep}}
\newcommand{\ul}{\setlength{\itemsep}{0pt}\begin{itemize}}
\newcommand{\edl}{\end{description}\setlength{\itemsep}{-\parsep}}
\newcommand{\eul}{\end{itemize}\setlength{\itemsep}{-\parsep}}

\newcommand{\commentout}[1]{}

\newcommand{\bi}{\begin{itemize}}
\newcommand{\ei}{\end{itemize}}
\newcommand{\be}{\begin{enumerate}}
\newcommand{\ee}{\end{enumerate}}

\newcommand{\onl}{\mathcal{ONL}_n}
\newcommand{\onlmin}{\mathcal{ONL}_n^-}

\newcommand{\edash}{{e'}}
\newcommand{\estar}{{e^*}}
\newcommand{\ebullet}{{e^\bullet}}

\newcommand{\njknow}{\mathitbf{N}_j}
\newcommand{\afive}{\mathbf{A5}_n}
\newcommand{\axioms}{{AX}_n}

\newcommand{\kffn}{{\textsf{K45}}_{n}}

\newcommand{\eakplus}{e_a^{k+1}}
\newcommand{\estarbk}{\estar_b^k}

\newcommand{\kstructure}{\( k \)-structure}

\newcommand{\no}{\noindent}

\newcommand{\kjmodel}{$(k,j)$-model}
\newcommand{\lan}{\langle}
\newcommand{\ran}{\rangle}

\newcommand{\canon}{M^c}
\newcommand{\acanon}{\mathcal{K}^c_a}
\newcommand{\icanon}{\mathcal{K}^c_i}
\newcommand{\bcanon}{\mathcal{K}^c_b}

\newcommand{\oknow}{\mathitbf{O}_i}
\newcommand{\aknow}{\mathitbf{L}_{a}} 
\newcommand{\aoknow}{\mathitbf{O}_{a}} 
\newcommand{\bknow}{\mathitbf{L}_{b}} 
\newcommand{\boknow}{\mathitbf{O}_{b}}

\newcommand{\know}{\mathitbf{L}_i}

\newcommand{\loneknow}{\mathitbf{L}}
\newcommand{\noneknow}{{\mathitbf{N}}}
\newcommand{\nknow}{\mathitbf{N}_i}
\newcommand{\naknow}{\mathitbf{N}_a}

\newcommand{\nbknow}{\mathitbf{N}_b}
\newcommand{\worldsp}{{\worlds}_p}

\newcommand{\ebone}{e_b^1}

\newcommand{\bel}{\varphi_{a0}}
\newcommand{\nbel}{\varphi_{aj}}
\newcommand{\nbelone}{\varphi_{a1}}

\newcommand{\most}{\psi_{a0}}
\newcommand{\nmost}{\psi_{aj}}

\newcommand{\canaObj}{\textit{Obj}_a} 
\newcommand{\sat}{\textit{Sat}}
\newcommand{\ebkdashdash}{e_b^{k''}} 

\newcommand{\onlplus}{\onl^+}
\newcommand{\axiomsplus}{\axioms^+}
\newcommand{\onlk}{\onl^t}
\newcommand{\onlp}{{\onl^+}}
\newcommand{\onlpk}{{\onl^+}^t}

\newcommand{\onlkplus}{\onl^{t+1}}
\newcommand{\axiomsk}{\axioms^t}

\newcommand{\axiomskplus}{\axioms^{t+1}}

\newcommand{\worlds}{\mathcal{W}}

\newcommand{\iobj}{\textit{obj}_i^+}

\newcommand{\aobj}{\textit{obj}_a^+}

\newcommand{\bobj}{\textit{obj}_b^+}

\newcommand{\aObj}{{O}{bj}^+_a}
\newcommand{\bObj}{{O}{bj}^+_b}

\newcommand{\ebjdash}{e_b^{j'}}
\newcommand{\eaone}{e_a^1}

	\newcommand{\estarbkmin}{\estar_b^{k-1}} 
\newcommand{\eakdash}{e_a^{k'}}
\newcommand{\mstar}{M^*}

\newcommand{\estarak}{\estar_a^k}

\newcommand{\estarbj}{\estar_b^j}

\newcommand{\wstar}{w^{*}}
\newcommand{\wstars}{{w^{\bullet}}}
\newcommand{\eatwo}{e_a^2}

\newcommand{\dl}{\llbracket}
\newcommand{\dr}{\rrbracket}

\newcommand{\jknow}{\mathitbf{L}_j}
\newcommand{\ebkdashmin}{e_b^{k'-1}}

\newcommand{\ebtwok}{e_b^{2k}}
\newcommand{\estarsbtwok}{\ebullet^{2k}_b}
\newcommand{\val}{\textit{Val}}

\newcommand{\eak}{e_a^k}
\newcommand{\ebkmin}{e_b^{k-1}}
\newcommand{\wrt}{\emph{wrt.}~}
\newcommand{\standardnames}{\mathcal{N}}

\newcommand{\wdl}{{\dl w \dr}}
\newcommand{\wdash}{{\dl w' \dr}}
\newcommand{\ew}{{e_{\dl w \dr}}}
\newcommand{\edashbkmin}{\edash_b^{k-1}}
\newcommand{\ewdash}{{e_{\dl w' \dr}}}
\newcommand{\ewdashdash}{e_{\dl w'' \dr}}

\newcommand{\ewak}{{\ew}_a^k}

\newcommand{\ewbj}{{\ew}_b^j}

\newcommand{\ewdashbkmin}{{\ewdash}_b^{k-1}}
\newcommand{\ewdashdashakminmin}{{\ewdashdash}_a^{k-2}}

\newcommand{\ebj}{e_b^j}
\newcommand{\sit}{\eak, \ebj, w}

\newcommand{\shortv}{\commentout}

\shortv{\linespread{1.02}}

\begin{document}

\begin{titlepage}
\title{{Multi-Agent Only-Knowing Revisited}} 

\author{
Vaishak Belle and Gerhard Lakemeyer \\
Dept. of Computer Science, \\ 
RWTH Aachen, \\ 
52056 Aachen, Germany \\ 
\{belle,gerhard\}@cs.rwth-aachen.de 
}   
\date{}
\maketitle

\begin{abstract}
 Levesque introduced the notion of only-knowing to precisely capture the beliefs of a knowledge base. He also showed how only-knowing can be used to formalize non-monotonic behavior within a monotonic logic. Despite its appeal, all attempts to extend only-knowing to the many agent case have undesirable properties. A belief model by Halpern and Lakemeyer, for instance, appeals to proof-theoretic constructs in the semantics and needs to axiomatize validity as part of the logic. It is also not clear how to generalize their ideas to a first-order case. In this paper, we propose a new account of multi-agent only-knowing which, for the first time, has a natural possible-world semantics for a quantified language with equality. We then provide, for the propositional fragment, a sound and complete axiomatization that faithfully lifts Levesque's proof theory to the many agent case. We also discuss comparisons to the earlier approach by Halpern and Lakemeyer. 
\end{abstract}
\thispagestyle{empty}
\end{titlepage}

\section{Introduction} 
\label{sec:introduction}

Levesque's notion of only-knowing is a single agent monotonic logic that was
proposed with the intention of capturing certain types of nonmonotonic
reasoning. Levesque (\citeyear{77758}) already showed that there is a close connection to
Moore's (\citeyear{2781}) autoepistemic logic (AEL). Recently, Lakemeyer and Levesque
(\citeyear{lakemeyer2005only}) showed that only-knowing can be adapted to capture default logic as
well. The main benefit of using Levesque's logic is that, via simple semantic
arguments, nonmonotonic conclusions can be reached without the use of
meta-logical notions such as fixpoints \cite{330786,levesque2001logic} . Only-knowing is then naturally of
interest in a many agent context, since agents capable of non-trivial
nonmonotonic behavior should believe other agents to also be equipped with
nonmonotonic mechanisms. For instance, if all that Bob knows is that Tweety is
a bird and a default that birds typically fly, then Alice, if she knows all
that Bob knows, concludes that Bob believes Tweety can fly.\footnote{We use
the terms "knowledge" and "belief" interchangeably in the paper.} Also, the
idea of only-knowing a collection of sentences is useful for modeling the beliefs of a 
knowledge base (KB), since sentences that are not logically entailed by the KB
are taken to be precisely those not believed. If many agents are involved, and
suppose Alice has some beliefs on Bob's KB, then she could capitalize on Bob's
knowledge to collaborate on tasks, or plan a strategy against him.

As a logic, Levesque's construction is unique in the sense that in addition to
a classical epistemic operator for belief, he introduces a modality to denote
what is \emph{at most} known. This new modality has a subtle relationship to
the belief operator that makes extensions to a many agent case non-trivial.
Most extensions so far make use of arbitrary Kripke structures, that already
unwittingly discard the simplicity of Levesque's semantics. They also have
some undesirable properties, perhaps invoking some caution in their usage. For
instance, in a canonical model (Lakemeyer 1993), certain types of epistemic
states cannot be constructed. In another Kripke approach (Halpern 1993), the
modalities do not seem to interact in an intuitive manner. Although an approach
by Halpern and Lakemeyer (\citeyear{1029713}) does indeed successfully model
multi-agent only-knowing, it forces us to have the semantic notion of validity
directly in the language and has proof-theoretic constructs in the semantics
via maximally consistent sets. Precisely for this reason, that proposal is not
natural, and it is matched with a proof theory that has a set of new axioms to
deal with these new notions. It is also not clear how one can extend their
semantics to the first-order case. Lastly, an approach by Waaler (\citeyear{DBLP:conf/aiml/Waaler04}) avoids such an axiomatization of validity, but the model theory also has problems \cite{DBLP:conf/tark/WaalerS05}. Technical discussions on their semantics are deferred to later.

The goal of this paper is to show that there is indeed a natural semantics for
multi-agent only-knowing for the quantified language with equality. For the
propositional subset, there is also a sound and complete axiomatization that
faithfully generalizes Levesque's proof theory.\footnote{The proof theory for
a quantified language is well known to be \emph{incomplete} for the single
agent case. It is also known that any complete axiomatization cannot be
\emph{recursive} \cite{204824,levesque2001logic}.} We also differ from Halpern
and Lakemeyer in that we do not enrich the language any more than necessary
(modal operators for each agent), and we do not make use of canonical Kripke
models. And while canonical models, in general, are only workable semantically
and can not be used in practice, our proposal has a computational appeal to
it. We also show that if we do enrich the language with a modal operator for
\emph{validity}, but only to establish a common language with \cite{1029713},
then we agree on the set of valid sentences. Finally, we obtain a
first-order multi-agent generalization of AEL, defined solely using notions of
classical logical entailment and theoremhood.

The rest of the paper is organized as follows. We review Levesque's notions,\footnote{There are other notions of "all I know", which will not be discussed here \cite{101989,ben1989all}. Also see \cite{330786}.} and 
define a semantics with so-called \( k \)-\emph{structures}. We then compare
the framework to earlier attempts. Following that, we introduce a sound and
complete axiomatization for the propositional fragment. In the last sections, we sketch the multi-agent (first-order) generalization of AEL, and prove that \( k \)-structures and \cite{1029713} agree on valid sentences, for an enriched language. Then, we conclude and end.

\newcommand{\ol}{\mathcal{ONL}}

\section{The \( k \)-structures Approach} 
\label{sec:the_logic}

The non-modal part of Levesque's logic\footnote{We name the logic following \cite{1029713} for ease of comparisons later on. It is referred to as \( \mathcal{OL} \) in \cite{204824,levesque2001logic}.} \( \ol \) consists of standard
first-order logic with \( = \) and a countably infinite set of standard names \(
\standardnames \).\footnote{More precisely, we have logical connectives \( \lor
  \), \( \forall \) and \( \neg \). Other connectives are taken for their usual syntactic
  abbreviations.}  To keep matters simple, function symbols are not considered
in this language.  We call a predicate other than \( = \), applied to
first-order variables or standard names, an \emph{atomic} formula. We write \(
\alpha^x_n \) to mean that the variable \( x \) is substituted in \( \alpha \) by a standard
name. If all the variables in a formula \( \alpha \) are substituted by standard
names, then we call it a \emph{ground} formula. Here, a world is simply a set of ground atoms, and the semantics is defined over the set of all possible worlds \( \worlds \). 
%
The standard names are thus \emph{rigid designators}, and denote
  precisely the same entities in all worlds. \( \ol \) also has two modal
operators: \( \loneknow \) and \( \noneknow \). While \( \loneknow \alpha \) is
to be read as "at least \( \alpha \) is known", \( \noneknow \alpha \) is to be
read as "at most \( \neg \alpha \) is known". A set of possible worlds is
referred to as the agent's \emph{epistemic state} \( e \). Defining a model to
be the pair \( (e,w) \) for \( w \in \worlds \), components of \( \ol \)'s
meaning of truth are: \begin{enumerate}
	\item \( e,w\models p \) iff \( p \in w \) and \( p \) is a ground atom,

	\item \( e,w \models (m = n) \) iff \( m \) and \( n \) are identical standard names, 
	
	\item \( e, w\models \neg \alpha \) iff \( e,w\not\models \alpha \), 

	\item \( e, w\models \alpha \lor \beta \) iff \(  e,w \models \alpha \) or \( e,w\models \beta \),

	\item \( e,w\models \forall x.~\alpha \) iff \( e,w\models \alpha^x_n \) for all standard names \( n \),  

	\item \( e,w \models \loneknow \alpha \) iff for all \( w' \in e \), \( e, w' \models \alpha \), and 

	\item \( e,w\models \noneknow \alpha \) iff for all \( w' \not\in e \), \( e,w' \models \alpha \). 
\end{enumerate}

\no The main idea is that \( \alpha \) is (at least) believed iff it is true at all worlds considered possible, while (at most) \( \alpha \) is believed to be false iff it is true at all worlds considered \emph{impossible}. So, an agent is said to only-know \( \alpha \), syntactically expressed as \( \loneknow \alpha \land \noneknow \neg \alpha \), when worlds in \( e \) are precisely those where \( \alpha \) is true. Halpern and Lakemeyer (\citeyear{1029713}) underline three features of the semantical framework of \( \ol \), the intuitions of which we desire to maintain in the many agent setting: \begin{enumerate}
	\item Evaluating \( \noneknow\alpha \) does \emph{not affect} the epistemic possibilities.  Formally, in \( \ol \), after evaluating formulas of the form \( \noneknow\alpha   \) the agent's epistemic state is still given by \( e \). 
	
	\item A union of the agent's possibilities, that evaluate \( \loneknow \),
and the impossible worlds that evaluate \( \noneknow \), is \emph{fixed} and
\emph{independent} of \( e \), and is the set of all \emph{conceivable}
states. Formally, in \( \ol \), \( \loneknow\alpha \) is evaluated \wrt worlds
\( w\in e \), and \( \noneknow\alpha \) is evaluated \wrt worlds \( w\in
\worlds -e \); the union of which is \( \worlds \). The intuition is that the
exact complement of an agent's possibilities is used in evaluating \(
\noneknow \).

	\item Given any set of possibilities, there is always a model where \emph{precisely} this set is the epistemic state. Formally, in \( \ol \), any subset of \( \worlds \) can be defined as the epistemic state. 
\end{enumerate}

\no Although these notions seem clear enough in the single agent case, generalizing them to the many agent case is non-trivial \cite{1029713}. We shall return to analyze the features shortly. Let us begin by extending the language. Let \( \ol_n \) be a first-order modal language that enriches the non-modal subset of \( \ol \) with modal operators \( \know \) and \( \nknow \) for \( i = a,b \). For ease of exposition, we only have two agents \( a \) (Alice) and \( b \) (Bob). Extensions to more agents is straightforward. We freely use \( \oknow \), such that \( \oknow \alpha \) is an abbreviation for \( \know \alpha \land \nknow \neg \alpha	 \), and is read as "all that \( i \) knows is \( \alpha \)". Objective and subjective formulas are understood as follows. 

\begin{definition} The \( i \)-depth of a formula \( \alpha \), denoted \( |\alpha|_i \), is defined inductively as (\( \Box_i \) denotes \( \know \) or
 \( \nknow \)):	\begin{enumerate}
		\item \( |\alpha|_i = 1 \) for atoms,               

		\item \( |\neg \alpha|_i = |\alpha|_i \),

		\item \( |\forall x.~\alpha|_i = |\alpha|_i \), 


		\item \( |\alpha \lor \beta|_i = \textrm{max}(|\alpha|_i, |\beta|_i) \), 

		\item \( |\Box_i \alpha|_i = |\alpha|_i \) , 

		\item \( |\Box_j\alpha|_i = |\alpha|_j + 1 \), for \( j \neq i \)
	\end{enumerate}
	A formula has a depth \( k \) if max(\( a \)-depth,\( b \)-depth) = \( k \). A formula is called \( i \)-objective if all epistemic operators which do not occur within the scope of another epistemic operator are of the form $\Box_j$ for $i\ne j$. A formula is called \( i \)-subjective if every atom is in the scope of an epistemic operator and all epistemic operators which do not occur within the scope of another epistemic operator are of the form $\Box_i.$  
\end{definition}

 For example, a formula of the form \( \aknow \bknow \aknow p \lor \bknow q\) has a depth of \( 4 \), a \( a \)-depth of \( 3 \) and a \( b \)-depth of \( 4 \).  \( \bknow q \) is both \( b \)-subjective and \( a \)-objective. A formula is called \emph{objective} if it does not mention any modal operators. A formula is called \emph{basic} if it does not mention any \( \nknow \) for \( i = a,b \). We now define a notion of epistemic states using \( k \)-\emph{structures}. The main intuition is that we keep separate the worlds Alice believes from the worlds she considers Bob to believe, to depth \( k \). 

%
%
%
%

\begin{definition} A \kstructure{} (\( k \geq 1 \)), say \( e^k \), for an agent is defined inductively as: \begin{itemize} 
	
	\item[$-$] \( e^1 \subseteq \worlds \times\{ \{\}\}, \)
	
	\item[$-$] \( e^k \subseteq \worlds \times \mathbb{E}^{k-1},  \) where \( \mathbb{E}^{m} \) is the set of all \( m\)-structures.

\end{itemize}

\end{definition}

\no A \( e^1 \) for Alice, denoted as \( \eaone \), is intended to represent a set of worlds \( \{ \lan w, \{\} \ran, \ldots \} \). A \( e^2 \) is of the form \( \{ \lan w, e_b^1 \ran, \lan w', {e'}_b^1 \ran, \ldots \} \), and it is to be read as "at \( w \), she believes Bob considers worlds from \( e_b^1 \) possible but at \( w' \), she believes Bob to consider worlds from \( {e'}_b^1 \) possible". This conveys the idea that Alice has only partial information about Bob, and so at different worlds, her beliefs about what Bob knows differ. We define a \( e^k \) for Alice, a \( e^j \) for Bob and a world \( w \in \worlds \) as a \( (k,j) \)-model \( (\eak, \ebj,w) \). Only sentences of a maximal \( a \)-depth of \( k \), and a maximal \( b \)-depth of \( j \) are interpreted \wrt{}a \( (k,j) \)-model. The complete semantic definition is: 

\begin{enumerate}
	\item \( \sit \models p \) iff \( p \in w \) and \( p \) is a ground atom, 

	\item \( \sit \models (m = n) \) iff \( m,n \in \standardnames \) and are identical, 
\item \( \sit \models \neg \alpha \) iff \( \sit \not\models \alpha \), 
\item \( \sit \models \alpha \vee \beta \) iff  \( \sit \models \alpha \) or \( \sit \models \beta \), 

\item \( \sit \models \forall x.~\alpha \) iff \( \sit \models \alpha^x_n \) for all \( n \in \standardnames \),
\item \( \sit \models \aknow \alpha \) iff for all \( \lan w', {\ebkmin} \ran \in {\eak} \),\\ \( {\eak}, {\ebkmin}, w' \models \alpha \), 
\item\label{i:nknow} \( \sit \models \naknow \alpha \) iff for all \( \lan w', {\ebkmin} \ran \not\in {\eak} \),\\ \( {\eak}, {\ebkmin}, w' \models \alpha \)
\end{enumerate} 

\no And since \( \aoknow \alpha  \) syntactically denotes \( \aknow \alpha \land \naknow \neg \alpha \), it follows from the semantics that
\begin{enumerate}
	\item[8.] \( \sit \models \aoknow \alpha \) iff for all worlds \( w' \),  
	for all \( e^{k-1} \) for Bob, 
	\( \lan w', {\ebkmin} \ran \in {\eak} \) iff \( {\eak}, {\ebkmin}, w' \models \alpha \)
\end{enumerate}

\no (The semantics for \( \bknow\alpha \) and \( \nbknow \alpha \) are given analogously.) A formula \( \alpha \) (of \( a \)-depth of \( k \) and of 
\( b \)-depth of \( j \)) is \emph{satisfiable} iff there is a \( (k,j)
\)-model such that \( \eak, \ebj,w \models \alpha \). The formula is
\emph{valid} (\( \models \alpha \)) iff \( \alpha \) is true at all \( (k,j) \)-models. Satisfiability is extended to a set of formulas \( \Sigma \) (of maximal \( a,b \)-depth of \( k,j \)) in the manner that there is a \( (k,j) \)-model \( \sit \) such that \( \sit \models \alpha' \) for every \( \alpha' \in \Sigma \). We write \( \Sigma \models \alpha \) to mean that for every \(
(k,j) \)-model \( \sit \), if \( \sit \models \alpha' \) for all \( \alpha' \in
\Sigma \), then \( \sit \models \alpha \). 

\newcommand{\edak}{{e_a}{\downarrow}_k^{k'}}
\newcommand{\edbkmin}{{e_b}{\downarrow}_{k-1}^{k'-1}}
\newcommand{\edakzero}{{e_a}{\downarrow}_0^{k'}}

Validity is not affected if models of a depth greater than that needed are used. This is to say, if \( \alpha \) is true \emph{wrt.}~all \( (k,j) \)-models, then \( \alpha \) is true \emph{wrt.}~all \( (k',j') \)-models for  \( k'\geq k,j'\geq j \). We obtain this result by constructing for every \( \eakdash \), a \( k \)-structure \( \edak \), such that they agree on all formulas of maximal \( a \)-depth \( k \). Analogously for \( \ebjdash \). 

\newcommand{\edakone}{{e_a}{\downarrow^{k'}_1}}

\begin{definition} Given \( \eakdash \), we define \( \edak \) for
  $k'\ge k \geq 1$: \begin{enumerate} 

  \item  \( {e_a}{\downarrow^1_1} = \eaone \),

  \item \( \edakone = \{ \lan w, \{\} \ran \mid \lan w, \ebkdashmin \ran \in \eakdash \} \), 
	\item \( \edak = \{ \lan w, \edbkmin \ran \mid \lan w, \ebkdashmin \ran \in \eakdash \} \). 
\end{enumerate} 

\end{definition}

\newcommand{\edbj}{{e_b}{\downarrow}^{j'}_j}
\newcommand{\edakmin}{{e_a}{\downarrow}_{k-1}^{k'}}
\newcommand{\edbjmin}{{e_b}{\downarrow^{k'-1}_{j-1}}}
\newcommand{\edakdone}{{e_a}{\downarrow_1^{k'}}}
\newcommand{\edbjdone}{{e_b}{\downarrow_1^{k'-1}}}
\newcommand{\edakdtwo}{{e_a}{\downarrow_2^{k'}}}
\newcommand{\edbjdtwo}{{e_b}{\downarrow_2^{j'}}}

\begin{lemma}\label{lem:satisfiability_higher_structures} For all formulas \( \alpha \) of maximal \( a,b \)-depth of \( k,j \), \( \eakdash, \ebjdash, w \models \alpha \) iff \( \edak, \edbj, w \models \alpha \), for \( k' \geq k,j'\geq j \).
\end{lemma}


\begin{proof} By induction on the depth of formulas.~The proof immediately holds for atomic formulas, disjunctions and negations since we have the same world \( w \). Assume that the result holds for formulas of \( a,b \)-depth \( 1 \). Let \( \alpha \) such a formula, and suppose \( \eakdash, \ebjdash, w\models \aknow \alpha \) (where \( \aknow \alpha \) has \( a,b \)-depth of \( 1,2 \)). Then, for all \( \lan w', \ebkdashmin \ran \in \eakdash \), \( \eakdash, \ebkdashmin, w'\models \alpha \) iff (by induction hypothesis) \( \edakdone, \edbjdone, w'\models \alpha \) iff \( \edakdtwo, \{\}, w\models \aknow \alpha \). By construction, we also have \( \edakdone, \{\},w\models \aknow \alpha \). Lastly, since \( \aknow \alpha  \) is \( a \)-subjective, \( b \)'s structure is irrelevant, and thus, \( \edakdone, \edbjdtwo, w\models \aknow\alpha \).
	
For the reverse direction, suppose \( \edakdone, \edbjdtwo, w\models \aknow \alpha \). Then for all \( w' \in \edakdone \), \( \edakdone, \{\}, w' \models \alpha \) iff (by construction) for all \( \lan w', \ebkdashmin \ran \in \eakdash \), \( \eakdash, \ebkdashmin, w' \models \alpha \) iff \( \eakdash, \{\}, w \models \aknow \alpha \). Since \( b \)'s structure is irrelevant, we have \( \eakdash, \ebjdash, w \models \aknow \alpha  \). The cases for \( \bknow \alpha \), \( \naknow \alpha \) and \( \nbknow \alpha \) are completely symmetric.    
	\qed \end{proof}

\begin{theorem} For all formulas \( \alpha \) of \( a,b \)-depth of \(
  k,j \), if \( \alpha \) is true at all \( (k,j) \)-models, then \( \alpha \)
  is true at all \( (k',j') \)-models with $k'\ge k$ and $j'\ge j$. 
	

\end{theorem}

\begin{proof} Suppose \( \alpha \) is true at all \( (k,j) \)-models. Given any \( (k',j') \)-model, by assumption \( \edak, \edbj, w\models \alpha \) and by Lemma \ref{lem:satisfiability_higher_structures}, \( \eakdash, \ebjdash, w\models \alpha \). 
	\qed \end{proof}

\no Knowledge with \( k \)-structures satisfy \emph{weak} \( \textsf{S5} \) properties, and the Barcan formula \cite{nla.cat-vn2004742}. 
\begin{lemma} If \( \alpha \) is a formula, the following are valid \wrt models of appropriate depth (\( \Box_i \) denotes \( \know \) or \( \nknow \)): \begin{enumerate}
	\item \( \Box_i \alpha \land \Box_i(\alpha \supset \beta) \supset \Box_i \beta \), 

	\item \( \Box_i \alpha \supset \Box_i\Box_i \alpha \), 

	\item \( \neg \Box_i \alpha \supset \Box_i \neg \Box_i \alpha \), 

	\item \( \forall \mathitbf{x}.~\Box_i \alpha \supset \Box_i(\forall \mathitbf{x}.~\alpha) \). 
\end{enumerate}

\end{lemma}      

\begin{proof} The proofs are similar. For item 3, wlog let \( \Box_i \) be \( \aknow \). Suppose \( \eak, \ebj, w \models \neg \aknow \alpha \). There is some \( \lan w',\ebkmin \ran \in \eak \) such that \( \eak, \ebkmin, w' \models \neg \alpha \). Let \( w'' \) be any world such that \( \lan w'', \edashbkmin \ran \in \eak \). Then, \( \eak, \edashbkmin, w'' \models \neg \aknow \alpha \). Thus, \( \eak, \ebj, w\models \aknow \neg \aknow \alpha \). The case of \( \naknow \) is analogous. 
	\qed \end{proof}

\newcommand{\true}{\textit{true}}

\no Before moving on, let us briefly reflect on the fact that $k$-structures
have finite depth. So suppose \( a \) only-knows KB, of depth \( k \). Using
\( k \)-structures allows us to reason about what is believed, up to depth \(
k \). Also, if we construct epistemic states from \( k' \)-structures where \(
k'\geq k \), then the logic correctly captures non-beliefs beyond the depth \(
k \). To illustrate, let \( \true \) (depth \( 1 \)) be all that \( a \)
knows. Then, it can easily be shown that both the sentences \( \aoknow(\true)
\supset \neg \aknow\neg \bknow \alpha \) and \( \aoknow(\true) \supset \neg
\aknow \bknow \alpha \) are valid sentences in the logic, by considering any
\( e^2 \) (and higher) for \( a \). For most purposes, this restriction of
having a parameter \( k \) seems harmless in the sense that agents usually
have a finite knowledge base with sentences of some maximal depth $k$ and they
should not be able to conclude anything about what is known at depths higher
than $k$, with one exception. If we were to include a notion of common
knowledge \cite{reasoning:about:knowledge}, then we would get entailments
about what is believed at arbitrary depths. With our current model, this
cannot be captured, but we are willing to pay that price because in return we
get, for the first time, a very simple possible-world style account of
only-knowing. Similarly,
we have nothing to say about (infinite) knowledge bases with unbounded depth.

\section{Multi-Agent Only-Knowing} 
\label{sec:multi_agent_only_knowing}

In this section, we return to the features of only-knowing discussed earlier
and verify that the new semantics reasonably extends them to the
multi-agent case.  We also briefly discuss earlier attempts at capturing these features.  Halpern (\citeyear{DBLP:conf/aaai/Halpern93}), Lakemeyer (\citeyear{Lakemeyer1993}), and Halpern and Lakemeyer (\citeyear{1029713}) independently attempted to
extend \( \ol \) to the many agent case.\footnote{For space reasons, we do not review all
aspects of these approaches.} There are some subtle differences in
their approaches, but the main restriction is they only allow a propositional
language. Henceforth, to make the comparison feasible, we shall also speak of
the propositional subset of \( \onl \) with the understanding that the
semantical framework is now defined for propositions (from an infinite set \(
\Phi \)) rather than ground atoms. 

The main component in these features is the notion of
\emph{possibility}.  
In the single agent case, each world represents a
possibility. Thus, from a logical viewpoint, a possibility is simply the set of
objective formulas true at some world. Further, the set of epistemic possibilities is given
by \( \{ \{\textrm{objective formulas true at}~w \} \mid w\in e\} \). Halpern and
Lakemeyer (\citeyear{1029713}) correctly argue that the appropriate generalization of
the notion of possibility in the many agent case are \( i \)-objective
formulas. Intuitively, a possible state of affairs according to \( a \)
include the state of the world (objective formulas), as well as what \( b \)
is taken to believe. 
%
%
The earlier attempts by Halpern and Lakemeyer use Kripke
structures with accessibility relations \( \mathcal{K}_i \) for each agent \( i \). Given a Kripke structure \( M \), the notion of possibility is defined as the set of \( i \)-objective formulas true at some Kripke world, and the set of epistemic possibilities is obtained from the \( i \)-objective formulas true at all \( i	 \)-accessible worlds. Formally, the set of epistemic possibilities true at \( (M,w) \), where \( w \) is a world in \( M \), is defined as \( \{ \iobj(M,w') \mid
w' \in \mathcal{K}_i(w) \} \), where \( \iobj(M,w') \) is a set consisting of
\( i \)-objective formulas true at \( (M,w') \).\footnote{The superscript \( +
\) denotes that the set includes non-basic formulas. Given \( X^+ \), we let
\( X = \{ \phi ~\textrm{is basic}~ \mid \phi \in X^+ \} \). } Although intuitive, note that, even
for the propositional subset of \( \ol \), a Kripke world is a completely
different entity from what Levesque supposes. Perhaps, one consequence is that
the semantic proofs in earlier approaches are very involved. In contrast, we
define worlds exactly as Levesque supposes. And, our notion of possibility is obtained from the set of \( a
\)-objective formulas true at each \( \lan w, \ebkmin \ran \) in \( \eak \).

\newcommand{\eajmin}{e_a^{j-1}}

\begin{definition} Suppose \(M = (\sit) \) is a \( (k,j) \)-model. \begin{enumerate}
	%
\item  let \( \iobj(M) = \{ \textrm{\( i \)-objective \( \phi \)} \mid M \models \phi \} \), 

\item  let \( \aObj({\eak}) = \{ \aobj(\{\}, {\ebkmin}, w) \mid \lan w, {\ebkmin} \ran \in {\eak} \} \), 

\item let \( \bObj(\ebj) = \{ \bobj(\eajmin, \{\}, w) \mid \lan w, \eajmin \ran \in \ebj \} \).
\end{enumerate}
	 
\end{definition}

\no All the \( a \)-objective formulas true at a model \( M \), essentially the objective formulas true \emph{wrt.}~\( w \) and the \( b \)-subjective formulas true w\emph{rt.}~\( \ebj \), are given by \( \aobj(M) \). Note that these formulas do not strictly correspond to \( a \)'s possibilities. Rather, we define \( \aObj \) on her epistemic state \( \eak \), and this gives us  all the \( a \)-objectives formulas that \( a \) considers possible. We shall now argue that the intuition of all of Levesque's properties is maintained.\footnote{It is interesting to note that such a formulation of Levesque's properties is not straightforward in the first-order case. That is, for the quantified language, it is known that there are epistemic states that can not be  characterized using only objective formulas \cite{levesque2001logic}. Thus, it is left open how one must correctly generalize the features of first-order \( \ol \). } \\
 
\no \textbf{Property 1.}~In the single agent case, this property ensured that an agent's epistemic possibilities are not affected on evaluating \( \noneknow \). This is immediately the case here. Given a model, say \( (\sit) \), \( a \)'s epistemic possibilities are determined by \( \aObj({\eak}) \). To evaluate \( \naknow\alpha \), we consider all models \(  ({\eak}, {\ebkmin}, w' ) \) such that \( \lan w', {\ebkmin} \ran \not\in {\eak}	  \). Again, \( a \)'s possibilities are given by \( \aObj({\eak}) \) for all these models, and does not change. \\[1ex]
\no \textbf{Property 2.} In the single agent case, this property ensured that evaluating \( \loneknow\alpha \) and \( \noneknow\alpha \) is always \wrt the set of all possibilities, and completely independent of \( e \). As discussed, in the many agent case, possibilities mean \( i \)-objective formulas and analogously, if \( \alpha \) is a possibility in \( a \)'s view, say an \( a \)-objective formula of maximal \( b \)-depth of \( k \), then we should interpret \( \aknow \alpha \) and \( \naknow \alpha \) \emph{wrt.}~all \( a \)-objective possibilities of max.~depth \( k \): the set of \( (k+1) \)-structures. Clearly then, the result is fixed and independent of the corresponding \( e^{k+1} \). The following lemma is a direct consequence of the definition of the semantics.

\begin{lemma} Let \( \alpha \) be a \( i \)-objective formula of \( j \)-depth \( k \), for \( j \neq i \). Then, the set of \( k\!+\!1 \)-structures that evaluate \( \know \alpha \) and \( \nknow \alpha \) is \( \mathbb{E}^{k+1} \). 

\end{lemma}

\no \textbf{Property 3.} The third property ensures that one can characterize
epistemic states from any set of \( i \)-objective formulas. Intuitively,
given such a set, we must have a model where \emph{precisely} this set is the
epistemic state. Earlier attempts at clarifying this property involved
constructing a \emph{set} of maximally \( \kffn \)-consistent sets of basic \(
i \)-objective formulas, and showing that there exist an epistemic state that precisely corresponds to this set. But, defining possibilities via \( \kffn \)
proof-theoretic machinery inevitably leads to some limitations, as we shall
see. We instead proceed semantically, and go beyond basic formulas. Let \(
\Omega \) be a satisfiable set of \( i \)-objective formulas, say of maximal
\( j \)-depth \( k \), for \( j\neq i \). Let \( \Omega' \) be a set obtained by adding a \( i
\)-objective formula \( \gamma \) of maximal \( j \)-depth \( k \) such that \(
\Omega' \) is also satisfiable. By considering all \( i \)-objective formulas
of maximal \( j \)-depth \( k \), let us construct \( \Omega' \), \( \Omega'' \), \(
\ldots \) by adding formulas iff the resultant set remains satisfiable. When
we are done, the resulting \( \Omega^\ast \) is what we shall call a maximally
satisfiable \( i \)-objective set.\footnote{A maximally satisfiable set is to
be understood as a semantically characterized \emph{complete} description of a
possibility, analogous to a proof theoretically characterized notion of
maximally consistent set of formulas.} Naturally, there may be many such sets
corresponding to \( \Omega \). We show that given a \emph{set} of maximally
satisfiable \( i \)-objective sets, there is a model where precisely this set
characterizes the epistemic state.

\begin{theorem}\label{thm:iset-theorem}
Let \( S_i \) be a set of maximally satisfiable sets of \( i \)-objective formulas, and \( \sigma \) a satisfiable objective formula. Suppose \( S_a \) is of max.~\( b \)-depth \( k-1 \) and \( S_b \) is of max.~\( a \)-depth \( j-1 \). Then there is a model \( \mstar = \lan \estarak, \estarbj, \wstar \ran \) such that \( \mstar \models \sigma \), \( S_a = \aObj(\estarak) \) and \( S_b = \bObj(\estarbj) \).

\end{theorem}

\begin{proof}
	 Consider \( S_a \). Each \( S' \in S_a \) is a maximally satisfiable \( a \)-objective set, and thus by definition, there is a \( k \)-structure \( \lan w', \ebkmin \ran \) such that \( \{\}, \ebkmin, w' \models S' \). Define such a set of \( k \)-structures \( \{ \lan w', \ebkmin \ran \} \), corresponding to each \( S' \in S_a \), and let this be \( \estarak \). It is immediate to verify that \( \aObj(\estarak) = S_a \). Analogously, for \( \estarbj \) using \( S_b \). Finally, there is clearly some world \( \wstar \) where \( \sigma \) holds. 
	\qed \end{proof}
	%
	
	
%

\subsection{On Validity} 
\label{sub:on_validity}



How does the semantics compare to earlier approaches? In particular, we are
interested in valid formulas. 
%
Lakemeyer (\citeyear{Lakemeyer1993}) proposes a semantics using \( \kffn \)-canonical models, but
he shows that the formula \(\neg \aoknow \neg \boknow p \) for any proposition \(
p \) is valid. Intuitively, it says that all that
Alice knows is that Bob does not only know \( p \), and as Lakemeyer argues,
the validity of \( \neg \aoknow \neg \boknow p \) is unintuitive. After all,
Bob could \emph{honestly} tell Alice that he does not only know \( p \). The
negation of this formula, on the other hand, is satisfiable in a Kripke
structure approach by Halpern (\citeyear{DBLP:conf/aaai/Halpern93}), called the \( i \)-set approach.\footnote{In his original formulation, Halpern (\citeyear{DBLP:conf/aaai/Halpern93}) constructs \emph{trees}. We build on discussions in \cite{1029713}.} It is
also satisfiable in the \( k \)-structure semantics. Interestingly, the \( i
\)-set approach and \( k \)-structures agree on one more notion. The formula
\( \aknow \bot \supset \neg \naknow \neg \boknow \neg \aoknow p \)~(\( \zeta \)) is valid in both, while \( \neg \zeta \) is satisfiable
\emph{wrt.}~Lakemeyer (\citeyear{Lakemeyer1993}). (It turns out that the validity of \( \zeta \)
in our semantical framework is implicitly related to the satisfiability of \(
\boknow \neg \aoknow p \), so this property is not unreasonable.) 


However, we immediately remark that the \( i \)-set approach and \( k \)-structures do not
share too many similarities beyond those presented above. In fact,~the \( i
\)-set approach does not truly satisfy Levesque's second property. For
instance, \( \naknow \neg \boknow p \land \aknow \neg \boknow p \)~(\(
\lambda \)) is satisfiable in Halpern (\citeyear{DBLP:conf/aaai/Halpern93}). Recall that, in this property, the
union of models that evaluate \( \nknow \alpha \) and \( \know \alpha \) must
lead to all conceivable states. So, the satisfiability of \( \lambda \) leaves open
the question as to why \( \boknow p \) is not considered since \( \neg \boknow
p \) is true at all conceivable states. We show that, in contrast, \(
\lambda \) is not satisfiable in the \( k \)-structures approach. Lastly, \cite{1029713}
involves enriching the language, the intuitions of which are perhaps best
explained after reviewing the proof theory, and so we defer discussions to
later.\footnote{An approach by
\cite{DBLP:conf/aiml/Waaler04,DBLP:conf/tark/WaalerS05} is also motivated by
the proof theory. Discussions are deferred.}                       
     


\begin{theorem}\label{lem:iset_validity} The following are properties of the semantics: \begin{enumerate}
	\item  \( \aoknow \neg \boknow p \), for any \( p \in \Phi \), is satisfiable.

	\item \( \models \aknow \bot \supset \neg \naknow \neg \boknow \neg \aoknow p \).

	\item \( \naknow\neg\boknow p \land \aknow \neg\boknow p \) is not satisfiable.  
\end{enumerate}

\end{theorem}

\begin{proof} \textbf{Item 1.} Let \( \worldsp = \{ w \mid w \models p \} \) and let \( E \) be all subsets of \( \worlds \) except the set \( \worldsp \). It is easy to see that if \( \ebone \in E \), then \( \{\}, \ebone, w \not\models  \boknow p \), for any world \( w \). Now, define a \( e^2 \) for \( a \) that has all of \( \worlds \times E \). Thus, \( \eatwo, 
\{\}, w \models \aoknow \neg \boknow p\).

\textbf{Item 2.}  Suppose \( \eak, \{\},
w\models \aknow \bot \) for any \( w \in \worlds \). Then, for all \(
\lan w', \ebkmin \ran \in \eak \), \( \eak, \ebkmin, w' \models \bot \), and thus, 
\( \eak = \{\} \). Suppose  now \( \eak, \{\}, w\models \naknow \neg \boknow \neg \aoknow p \). Then, \wrt all of \( \lan w', \ebkmin \ran \not\in\eak \) i.e.~all of \( \mathbb{E}^k \), \( \neg \boknow \neg \aoknow p \) must hold. That is, \( \neg \boknow \neg \aoknow p \) must be valid. From above, we know this is not the case.

\textbf{Item 3.}  Suppose \( \eak, \{\}, w \models \aknow \neg \boknow  p \), for any \( w \). Then, for all \( \lan w', \ebkmin \ran \in \eak \), \( \eak, \ebkmin, w' \models \neg \boknow p \). Since \( \boknow p \) is satisfiable, there is a \( \estarbkmin \) such that \( \{\}, \estarbkmin, \wstar \models \boknow p \), and  \( \lan \wstar, \estarbkmin \ran \not\in \eak \). Then, \( \eak, \{\}, w \models \neg \naknow \neg \boknow p \).              \qed \end{proof}

\no Thus, \( k \)-structures seem to satisfy our intuitions on the behavior of only-knowing. To understand why, notice that  \( \neg \aoknow \neg \boknow p \) and \( \lambda \) involve the nesting of \( \nknow \) operators. Lakemeyer (\citeyear{Lakemeyer1993}) makes an unavoidable technical commitment. A (\( i \)-objective) possibility is formally a maximally \( \kffn \)-consistent set of \emph{basic} \( i \)-objective formulas. The restriction to basic formulas is an artifact of a semantics based on the canonical model. Unfortunately, there is more to agent \( i \)'s possibility than just basic formulas. In the case of Halpern (\citeyear{DBLP:conf/aaai/Halpern93}), the problem seems to be that \( \nknow \) and \( \know \) do not interact naturally, and that the full complement of epistemic possibilities is not considered in interpreting \( \nknow \). In contrast, Theorem \ref{thm:iset-theorem} shows that we allow non-basic formulas and by using a strictly semantic notion, we avoid problems that arise from the proof-theoretic restrictions. And, since the semantics faithfully complies with the second property, \( \lambda \) is not satisfiable.

%
%

The natural question is if there are axioms that characterize the semantics. We begin, in the next section, with a proof theory by Lakemeyer (\citeyear{Lakemeyer1993}) that is known to be sound and complete for all attempts so far, but for a restricted language.

\section{Proof Theory} 
\label{sec:proof_theory}
In the single agent case, \( \ol \)'s proof theory consists of axioms of
propositional logic, axioms that treat \( \loneknow \) and \( \noneknow \) as
a classical belief operator in \( \textsf{K45} \), an axiom that allows us to use
\( \noneknow \) and \( \loneknow \) freely on subjective formulas, modus ponens (\( \mathbf{MP} \)) and
 necessitation (\( \mathbf{NEC} \)) for both \( \loneknow \) and \(
\noneknow \) as inference rules, and the following axiom:\footnote{Strictly speaking, this is not the proof theory introduced in~\cite{77758}, where an axiom replaces the inference rule \( \mathbf{NEC} \).
Here, we consider an equivalent formulation by Halpern and Lakemeyer~(\citeyear{1029713}).}  \begin{itemize}
	\item[] \( \mathbf{A5}. \) \( \noneknow \alpha \supset \neg \loneknow \alpha \)
	if \( \neg \alpha \) is a propositionally consistent\\\qquad \mbox{}\qquad \mbox{} objective formula.
\end{itemize}

\no As we shall see, only the axiom
\( \mathbf{A5} \) is controversial, since extending any objective \( \alpha \)
to \emph{any} \( i \)-objective \( \alpha \) is problematic. Mainly, the soundness of the axiom in the single agent case relies on propositional logic. But in the multi-agent case, since we go beyond propositional formulas establishing this consistency is non-trivial, and even circular. To this end, Lakemeyer~(\citeyear{Lakemeyer1993}) proposes to resolve this consistency by relying on the existing logic \( \kffn \). As a consequence, his proof theoretic formulation appropriately generalizes all of Levesque's axioms, except for \( \mathbf{A5} 	 \) where its application is restricted to only basic \( i \)-objective consistent formulas. We use \( \vdash \) to denote provability.                                     


\begin{definition} \( \onlmin \) consists of all formulas \( \alpha \) in \(
\onl \) such that no \( \njknow \) may occur in the scope of a \( \know \) or
a \( \nknow \), for \( i \neq j \).

\end{definition}

\no The following axioms, along with \( \mathbf{MP} \) and \( \mathbf{NEC} \) (for \( \know \) and \( \nknow \)) is an axiomatization that we refer to as \( \axioms \). \( \axioms \) is sound and complete for the canonical model and the \( i \)-set approach for formulas in \( \onlmin \). 

\begin{itemize}

	\item[]  \( \mathbf{A1}_n. \) All instances of propositional logic, 

	\item[]  \( \mathbf{A2}_n. \) \( \know(\alpha \supset \beta) \supset (\know \alpha \supset \know \beta) \), 

	\item[] \( \mathbf{A3}_n. \) \( \nknow(\alpha \supset \beta) \supset (\nknow \alpha \supset \nknow \beta) \), 

	\item[]   \( \mathbf{A4}_n. \) \( \sigma \supset \know \sigma \land \nknow \sigma \) for \( i \)-subjective \( \sigma \), 

	\item[] \( \afive.~\nknow \alpha \supset \neg \know \alpha  
	 \) if \( \neg \alpha \) is a \( \kffn \)-consistent\\\qquad \mbox{}\qquad \mbox{} \( i \)-objective basic formula.
\end{itemize}


\no Observe that, as discussed, the soundness of \( \afive \) is built on \( \kffn \)-consistency. Since our semantics is not based on Kripke structures, proving that every \( \kffn \)-consistent formula is satisfiable in some \( (k,j) \)-model is not immediate. We propose a construction called the \( (k,j) \)-\emph{correspondence model}. In the following, in order to disambiguate \( \worlds \) from Kripke worlds, we shall refer to our worlds as propositional valuations.

       




\begin{definition} The \( \kffn \) canonical model \( \canon = \lan \worlds^c, \pi^c, \acanon, \bcanon \ran \)  is defined as follows: \begin{enumerate}
	
		\item \( \worlds^c = \{ w \mid w \) is a (basic) maximally consistent set \( \} \) 

	\item for all \( p \in \Phi\) and worlds \( w \), \( \pi^c(w)(p) = \textrm{true} \) iff \( p \in w \) 

	\item \( (w,w') \in \icanon \) iff \( w\backslash \know \subseteq w' \), \( w\backslash \know = \{ \alpha \mid \know\alpha \in w\} \)
\end{enumerate}

\end{definition}

\begin{definition} Given \( \canon \), define a set of propositional valuations \( \worlds \) such that for each world \( w \in \worlds^c \), there is a valuation \( \dl w \dr \in \worlds \), \( \dl w \dr = \{ p \mid p \in w \} \). 
\end{definition}

\newcommand{\ewaone}{{e_{\dl w \dr}}_a^1}

\begin{definition}\label{defn:correspondence_model} Given  \( \canon \) and a world \( w \in \worlds^c \), construct a \kjmodel{} \( \lan \ewak, \ewbj, \dl w \dr \ran \) from valuations \( \worlds \) inductively: \begin{enumerate}
	\item \( \ewaone = \{ \lan \dl w' \dr, \{\} \ran \mid w' \in \acanon(w) \} \), 

	\item \( \ewak = \{ \lan \dl w' \dr, \ewdashbkmin \ran \mid w' \in \acanon(w) \} \),  

	\item[] where \( \ewdashbkmin = \{ \lan \dl w'' \dr, \ewdashdashakminmin \ran \mid w'' \in \bcanon(w') \} \).
\end{enumerate} 
	
%
%
\no Further, \( \ewbj \) is constructed analogously. Let us refer to this model as the \( (k,j) \)-correspondence model of \( (\canon, w) \). 
\end{definition}

\no Roughly, Defn.~\ref{defn:correspondence_model} is a construction of a \(
(k,j) \)-model that appeals to the accessibility relations in the canonical
model.\footnote{The construction is somewhat similar to the notion of \emph{generated submodels} of Kripke frames \cite{hughes1984companion}.} Thus, a \( \eaone \) for Alice \emph{wrt.}~\( w \) has precisely the
valuations of Kripke worlds \(w' \in \acanon(w) \). Quite analogously, a \(
\eak \) is a set \( \{ \lan \dl w' \dr, e^{k-1} \ran \} \), where \( w' \in
\acanon(w) \) as before, but \( e^{k-1} \) is an epistemic state for Bob and
hence refers all worlds \( w'' \in \bcanon(w') \). By a induction on the depth
of a \emph{basic} formula \( \alpha \), we obtain a theorem that \( \alpha \)
of maximal \( a,b \)-depth \( k,j \) is satisfiable at \( (\canon,w) \) iff
the \( (k,j) \)-correspondence model satisfies the formula.

%

%



\begin{theorem}\label{thm:basicformulas_canon_iff_kjmodel} For all basic formulas \( \alpha \) in \( \onlmin \) and of maximal \( a,b \)-depth of \( k,j \), \\
	\qquad \mbox{}\qquad \mbox{}\qquad \mbox{}\( \canon, w \models \alpha \) iff \( \ewak, \ewbj, \dl w \dr \models \alpha \). 

\end{theorem}
                                        
\begin{proof} By definition, the proof holds for propositional formulas,
disjunctions and negations. So let us say the result holds for formulas of \(
a,b \)-depth \( 1 \). Suppose now \( \canon,w \models \aknow\alpha \), where
\( \aknow \alpha \) has \( a,b \)-depth of \( 1,2 \). Then for all \( w'\in
\acanon(w) \), \( \canon, w'\models \alpha \) iff (by induction hypothesis) \(
\ewdash_a^1, \ewdash_b^1, \wdash \models \alpha \) iff \( \ew_a^2, \{\}, \wdl
\models \aknow \alpha \). By construction, we also have \( \ew_a^1, \{\}, \wdl \models
\aknow \alpha \). Since \( b \)'s structure is irrelevant, we get \( \ew_a^1,
\ew_b^2, \wdl \models \aknow \alpha \) proving the hypothesis. 


	 	For the other direction, suppose \( \ew_a^1, \ew_b^2, \wdl \models \aknow
\alpha \). For all \( \wdash \in \ew_a^1 \), \( \ew_a^1, \{\},
\wdash \models \alpha \) iff (by hyp.) \( \canon, w'\models \alpha \)
for all \( w' \in \acanon(w) \) iff \( \canon,w\models \aknow \alpha \).\qed \end{proof}

\begin{lemma}\label{lem:every_kffn_consistent_is_sat} Every \( \kffn \)-consistent basic formula \( \alpha \) is satisfiable wrt.~some \kjmodel. 

\end{lemma}

\begin{proof} 
	It is a property of the canonical model that every \( \kffn \)-consistent basic formula is satisfiable \emph{wrt.}~the canonical model. Supposing that the formula has a \( a,b \)-depth of \( k,j \) then from Thm \ref{thm:basicformulas_canon_iff_kjmodel}, we know there is at least the correspondence  \kjmodel{} that also satisfies the formula.
	\qed \end{proof}   
	
	                 \begin{theorem}\label{thm:soundness_onlmin}

		 For all \( \alpha \in \onlmin \), if \( \axioms \vdash \alpha \) then \( \models \alpha \).

	\end{theorem}   
	
	\begin{proof}		The soundness is easily shown to hold for \( \mathbf{A1}_n - \mathbf{A4}_n \). The soundness of \( \afive \) is shown by induction on the depth. Suppose \( \alpha \) is a propositional formula, and say \( \neg \alpha \) is a consistent  propositional formula (and hence \( \kffn \)-consistent). Then there is a world \( \wstar \) such that \( \{\},\{\}, \wstar\models \neg \alpha \). Given a \( \eak \), if \( \lan \wstar, \ebkmin \ran \in \eak \) for some \( \ebkmin \), then \( \eak, \{\}, w \models \neg \aknow \alpha \) for any world \( w \). If not, then \( \eak, \{\}, w\models \neg \naknow \alpha \). Thus,  \( \eak, \{\}, w\models \naknow \alpha \supset \neg \aknow \alpha \). Wlog, assume the proof holds for \( a\)-objective  formulas of max.~\( b \)-depth \( k-1 \).  Suppose now, \( \alpha   \) is such a formula, and \( \neg \alpha \) is \( \kffn \)-consistent. By Lemma \ref{lem:every_kffn_consistent_is_sat}, there is \( \lan \wstar, \estarbkmin \ran \), such that \( \{\}, \estarbkmin, \wstar \models \neg \alpha \). Again, if \( \lan \wstar, \estarbkmin \ran \in \eak \), then \( \eak, \{\}, w\models \neg \aknow \alpha \) and if not, then \( \eak, \{\}, w\models \neg \naknow \alpha \).
		\qed \end{proof}

\no We proceed with the completeness over the following definition, and lemmas. 

\begin{definition} A formula \( \psi \) is said to be \emph{independent} of the formula \( \phi \) wrt. an axiom system \( AX \), if neither \( AX \vdash \phi \supset \psi \) nor \( AX \vdash \phi \supset \neg \psi \). 

\end{definition}

\begin{lemma}[Halpern and Lakemeyer, 2001]\label{lem:existence_independent}
If \( \phi_1, \ldots, \phi_m \) are \( \kffn \)-consistent basic \( i
\)-objective formulas then there exists a basic \( i \)-objective formula \(
\psi \) of the form \( \jknow \psi' \) (\( j \neq i \)) that is independent of \( \phi_1,
\ldots, \phi_m \) wrt. \( \kffn \).

\end{lemma}

\begin{lemma}\label{lem:depth_of_independent_formula} In the lemma above, if \( \phi_i \) are \( i  \)-objective and of maximal \( j \)-depth \( k \) for \( j\neq i \), then there is a \( \psi \) of \( j \)-depth \( 2k+2 \). 

\end{lemma}

\begin{lemma}[Halpern and Lakemeyer, 2001]\label{lem:phi_or_psi_is_valid_basic} If \( \phi \) and \( \psi \) are \( i \)-objective basic formulas, and if \( \know \phi \land \nknow \psi \) is \( \axioms \)-consistent, then \( \phi \lor \psi \) is valid.  

\end{lemma}
                
\begin{lemma}[Halpern and Lakemeyer, 2001] Every formula \( \alpha \in \onl \) is provably equivalent to one in the normal form \emph{(}written below for \( n = \{a,b\} \)\emph{):}   \\[1ex]
\no \( \bigvee(\sigma \land \aknow \bel \land \neg \aknow \nbelone  \ldots \land \neg \aknow \varphi_{a{m_1}}  \land \bknow 	\varphi_{b0} \ldots \land \neg \bknow \varphi_{b{m_2}}  \land 
 \naknow \most \ldots \land \neg \naknow \psi_{a{n_1}}   \land \nbknow \psi_{b0} \ldots \land \neg \nbknow \psi_{bn_2} ) \) \\[1ex]
\no where \( \sigma \) is a propositional formula, and \( \varphi_{im} \) and \( \psi_{in} \) are \( i \)-objective. If \( \alpha \in \onlmin \), \( \varphi_{im} \) and \( \psi_{in} \) are basic.

\end{lemma}

\begin{theorem}\label{thm:completeness_onlmin} For all formulas \( \alpha \in \onlmin \), if \(\models \alpha \) then \( \axioms \vdash \alpha \). 

\end{theorem}


\begin{proof} 
	It is sufficient to prove that every \( \axioms \)-consistent
formula \( \xi \) is satisfiable \wrt some \kjmodel. If \( \xi \) is basic,
then by Lemma \ref{lem:every_kffn_consistent_is_sat}, the statement holds. If
\( \xi \) is not basic, then wlog, it can be
considered in the normal form: \\[1ex] 
	   \no \( \bigvee(\sigma \land \aknow \bel \land \neg \aknow \nbelone  \ldots \land \neg \aknow \varphi_{a{m_1}}  \land \bknow 	\varphi_{b0} \ldots \land \neg \bknow \varphi_{b{m_2}}  \land 
	 \naknow \most \ldots \land \neg \naknow \psi_{a{n_1}}   \land \nbknow \psi_{b0} \ldots \land \neg \nbknow \psi_{bn_2} ) \) \\[1ex]
\no where \( \sigma \) is a propositional formula, and \( \varphi_{im} \)
and \( \psi_{in} \) are \( i \)-objective and basic. Since \( \sigma \) is
propositional and consistent, there is clearly a world \( \wstar \) such that
\( \wstar \models \sigma \). We construct a \( k' \)-structure such that it
satisfies all the \( a \)-subjective formulas in the normal form above.
Following that, a \( j' \)-structure for all the \( b \)-subjective formulas
is constructed identically. The resulting \( (k',j') \)-model (with \( \wstar
\)) satisfies \( \xi \).


Let \( A \) be all \( \kffn \)-consistent formulas of the form \( \bel \land
\most \land \neg \nbel \) (for \( j \geq 1 \)) or the form \( \bel \land \most
\land \neg \nmost \). Let \( \gamma \) be independent of all formulas in \( A
\), as in Lemma \ref{lem:existence_independent} and
\ref{lem:depth_of_independent_formula}. Note that, while we take \( \xi
\) itself to be of maximal \( a,b \)-depth of \( k,j \), the depth of \( \bel,
\ldots \) being \( a \)-objective are of maximal \( b \)-depth \( k-1 \), and hence
\( \gamma \) is of \( b	 \)-depth \( 2k \) (Lemma
\ref{lem:depth_of_independent_formula}). Given a consistent set of
formulas, the standard Lindenbaum construction can be used to construct a
maximally consistent set of formulas, all of a maximal \( b \)-depth \( k-1 \). That is,
a formula is considered in the construction only if it has a maximal \( b \)-depth 
\( k-1 \). Now, let \( S_a \) be a set of all maximally consistent sets of
formulas, constructed by only considering formulas of maximal \( b \)-depth \( k-1 \), and containing \( \bel \land (\neg
\most \lor (\most \land \gamma)) \). Since each of these consistent sets are
basic and \( a \)-objective, they are satisfiable by
Lemma~\ref{lem:every_kffn_consistent_is_sat}.  Thus the sets \( S' \in S_a \) are
satisfiable \wrt \( 2k \)-structures \( \lan w, \ebtwok \ran \). Let \( k'
=2k+1 \). By constructing a \( k' \)-structure for Alice, say \( \eakdash \),
from each \( \lan w, \ebtwok \ran \) for every \( S' \in S_a \), we have that
\( \canaObj(\eakdash) = S_a \). We shall show that all the \( a \)-subjective
formulas in the normal form are satisfied \wrt \( \lan {\eakdash}, \{\}, \wstar
\ran \).

Since for all \( S' \in S_a \), we have \( \bel \in S' \) we get that \(
{\eakdash}, \{\}, \wstar \models \aknow \bel \). Now, since \( \aknow \bel
\land \neg \aknow \nbel \) is consistent, it must be that \( \bel \land \neg
\nbel \) is consistent. For suppose not, then \( \neg \bel \lor \nbel \) is
provable and thus, we have \( \bel \supset \nbel \). We then prove \( \aknow
\bel \supset \aknow \nbel \), and since we have \( \aknow \bel \) we prove \(
\aknow \nbel \), clearly inconsistent with \( \aknow \bel
\land\neg \aknow \nbel \). Now that
\( \bel \land \neg \nbel \) is consistent, we either have that \( \bel \land
\neg \nbel \land \most \) or \( \bel \land \neg \nbel \land \neg \most \) is
consistent. With the former, we also have that \( \bel \land \neg \nbel \land
\most \land \gamma \) is consistent. There are maximally consistent sets that
contain one of them, both of which contain \( \neg \nbel \). This means that,
\( {\eakdash}, \{\}, \wstar \models \neg \aknow \nbel \).


Now, consider some \( k' \)-structure \( \lan \wstars, \estarsbtwok \ran \not
\in {\eakdash} \). One of the following \( a \)-objective formulas must hold \wrt
this \( k' \)-structure: (a) \( \bel \land \most \), (b) \( \bel \land \neg
\most \), (c) \( \neg \bel \land \most \) or (d) \( \neg \bel \land \neg \most
\). It can not be (d), since \( \aknow \bel \land \naknow \most \) is
consistent, and this implies that \( \bel \lor \most \) is valid (by Lemma
\ref{lem:phi_or_psi_is_valid_basic}). It certainly cannot be (b), for it would
be in some \( S' \in S_a \). This leaves us with options (c) and (a), both of
which have \( \neg \most \). Since the \( k' \)-structure was arbitrary, we
must have for all \( \lan w, \ebtwok \ran \not\in {\eakdash} \), \( \{\},
{\ebtwok}, w \models \most \). Thus, \( {\eakdash}, \{\}, \wstar \models
\naknow \most \).

Finally, since \( \naknow \most \land \neg \naknow \nmost \) is consistent, it
must be that \( \most \land \neg \nmost \) is consistent. Further, either \(
\most \land \neg \nmost \land \bel \) or \( \most \land \neg \nmost \land \neg
\bel \) is consistent. If the former, then \( \most \land \neg \nmost \land
\bel \land \neg \gamma \) is also consistent. Let \( \beta \) be that which is
consistent. Note that \( \neg \beta \land (\bel \land (\neg \most \lor (\most
\land \gamma))) \) is consistent, and hence part of all \( S' \in S_a \). This
means that \( {\eakdash}, \{\}, \wstar \models \aknow(\neg \beta) \). But
since \( \beta \) itself is consistent, there is a \( k' \)-structure such
that \( \{\}, e_{(b,2k)}^\bullet, \wstars \models \beta \). And this \( k'
\)-structure can not be in \( {\eakdash} \). This means that \( {\eakdash},
\{\}, \wstar \models \neg \naknow \nmost \). Thus, all the \( a \)-subjective
formulas in the normal form above are satisfiable \wrt \( {\eakdash} \).
\qed \end{proof}   


\no  Now, observe that, although \( \aknow \bot \supset \neg \naknow \neg \boknow \neg \aoknow p \)~(\( \zeta \)) from Theorem \ref{lem:iset_validity} is valid, yet it is not derivable from \( \axioms \). In fact, the soundness result is easily extended to the full language \( \onl \). Then, the proof theory cannot be complete for the full language since there is \( \zeta \in \onl \) such that \( \not\vdash \zeta \) and \( \models \zeta \). Similarly, the validity of non-provable formulas \( \neg \aoknow \neg \boknow p \) and \( \zeta \) \emph{wrt.}~the canonical model and the \( i \)-set approach respectively, show that although \( \axioms \) is also sound for the full language in these approaches, it cannot be compelete. Mainly, axiom \( \afive \) has to somehow go beyond basic formulas. As Halpern and Lakemeyer (\citeyear{1029713}) discuss, the problem is one of circularity. We would like the axiom to hold for any \( \alpha \) such that it is a consistent \( i \)-objective formula, but to deal with consistency we have to clarify what the axiom system looks like. 

The approach taken by Halpern and Lakemeyer is to introduce \emph{validity} (and its dual satisfiability) directly into the language. Formulas in the new language, \( \onlplus \), are shown to be provably equivalent to \( \onl \). Some new axioms involving validity and satisfiability are added to the axiom system, and the resultant proof theory \( \axiomsplus \) is shown to be sound and complete for formulas in \( \onlplus \), \emph{wrt.}~an \emph{extended} canonical model. (An extended canonical model follows the spirit of the canonical model construction but by considering maximally \( \axiomsplus \)-consistent sets, and treat \( \know \) and \( \nknow \) as two independent modal operators.) So, one approach is to show that for formulas in the extended language the set of valid formulas overlap in the extended canonical model and \( k \)-structures. But then, as we argued, axiomatizing validity is not natural. Also, the proof theory is difficult to use. And in the end, we would still understand the axioms to characterize a semantics bridged on proof-theoretic elements. 

Again, what is desired is a generalization of Levesque's axiom \( \mathbf{A5} \), and nothing more. To this end, we propose a new axiom system, that is subtly related to the structure of formulas as are parameters \( k \) and \( j \). The axiom system has an additional \( t \)-axioms, and is to correspond to a sequence of languages \( 
\onlk \).\footnote{The idea was also suggested by a reviewer in \cite{1029713} for an axiomatic characterization of the extended canonical model, although its completeness was left open.} 


\begin{definition} Let \( \onl^1 = \onlmin \). Let \( \onlkplus \) be all Boolean combinations of formulas of \( \onlk \) and formulas of the form \( \know \alpha \) and \( \nknow \alpha \) for \( \alpha \in \onlk \). 

\end{definition}

\no It is not hard to see that \( \onlkplus \supseteq \onlk \). Note that \( t \)
here does not correspond to the depth of formulas. Indeed, a formula of the
form \( (\bknow\aknow)^{k+1}p \) is already in \( \onlmin \). Let \( \axiomskplus \)
be an axiom system consisting of \( \mathbf{A1}_n-\mathbf{A4}_n \), \( \mathbf{MP} \), \( \mathbf{NEC} \) and \( \afive^1-\afive^{t+1} \) defined inductively as: \begin{itemize}
	\item[]  \( \afive^1. ~\nknow \alpha	 \supset \neg \know \alpha	 \), if \( \neg \alpha \) is a \( \kffn \)-consistent \\ 
	\qquad \mbox{}\qquad \mbox{}\( i \)-objective basic formula. 

	\item[] \(  \afive^{t+1}. ~\nknow \alpha \supset \neg \know \alpha	 \), if \( \neg \alpha \in \onlk \), is \( i \)-objective,\\
	\qquad \mbox{}\qquad \mbox{} and consistent \emph{wrt.}~\( \mathbf{A1}_n-\mathbf{A4}_n \), $\afive^1-\afive^t$ . 
\end{itemize}

\begin{theorem}\label{thm:soundness_full_lang} For all \( \alpha \in \onlk \), if \( \axiomsk \vdash \alpha \) then \( \models \alpha \). 

\end{theorem}

\begin{proof} 
	We prove by induction on \( t \). The case of \( \axioms^1 \) is
identical to Theorem \ref{thm:soundness_onlmin}. So, for the induction hypothesis,
let us assume that \wrt \( \axiomsk \), if \( \axiomsk \vdash \beta \) for \(
\beta \in \onlk \) then \( \models\beta \). Now, suppose that \( \neg \alpha
\) is consistent \wrt \( \axiomsk \) and is \( a \)-objective. This implies
that \( \not \models \alpha \). Thus, there is some \( k \)-structure \( \lan
\wstar, \estarbk \ran \) such that \( \{\}, \estarbk, \wstar \models \neg
\alpha \). Suppose now \( \lan \wstar, \estarbk \ran \in {\eakplus} \) then \(
{\eakplus}, \{\}, w' \models \neg \aknow \alpha \) and if not then \( {\eakplus},
\{\}, w' \models \neg \naknow \alpha \). Thus, \( {\eakplus}, \{\}, w' \models
\naknow \alpha \supset \neg \aknow \alpha \), demonstrating the soundness of
\( \axiomskplus \).  
\qed \end{proof}

\no We establish completeness in a manner identical to Theorem \ref{thm:completeness_onlmin}, and thus it necessary to ensure that Lemma  \ref{lem:existence_independent}, \ref{lem:depth_of_independent_formula} and \ref{lem:phi_or_psi_is_valid_basic} hold for non-basic formulas. 

\begin{lemma} If \( \phi_1, \ldots, \phi_m \) are \( \axiomsk \)-consistent \( i \)-objective formulas, then there is a basic formula \( \psi \) of the form \( \jknow \psi \) (\( j \neq i \)) that is independent of \( \phi_1, \ldots, \phi_m \) \emph{\wrt }\( \axiomsk \).

\end{lemma}


\begin{proof} 
	Suppose that \( \phi_i \) are \( a \)-objective and of maximal \( b \)-depth \( k \). A formula \( \psi \) of the form \( (\bknow \aknow)^{k+1}p \) (where \( p \in \Phi \) is in the scope of \( k+1 \) \( \bknow \aknow \)) is shown to be independent of \( \phi_1, \ldots, \phi_m \). Let us suppose we can derive a \( \gamma \) of the form \( \bknow \aknow \bknow \aknow \ldots p \) of maximal depth \( k \), to show that neither \( \vdash \gamma \supset \psi \) nor \(  \vdash \gamma \supset \neg \psi \). Given any formula, the only axioms in \( \axiomsk \) that can introduce \( \gamma \) in the scope of modal operators is \( \mathbf{A4}_n \) and \( \afive^{t} \). Applying \( \mathbf{A4}_n \) gives \(\bknow \gamma \) or \(  \nbknow \gamma \), and then using the axiom again we have \( \bknow \bknow \gamma \) or \( \bknow \nbknow \gamma \). It is easy to see that the resulting formulas are clearly independent from \( \psi \). Applying \( \afive^t \) on the other hand, allows us to derive \( \vdash \gamma \supset \naknow \gamma \) or \(  \vdash \gamma \supset \neg \aknow \gamma \) (\( \gamma \) is consistent \emph{wrt.~}\( \axiomsk \) and hence also \emph{wrt.~}\( \afive^{t-1} \)). Again, we could show \( \vdash \gamma \supset \neg \bknow \neg \aknow \gamma \). Continuing this way, it might only be possible to derive \( \neg \bknow \neg \aknow \ldots \bknow \aknow \ldots p \) of depth \( 2k+2 \), that is indeed independent of \( \psi \).    
	\qed \end{proof}
	
	%

\begin{lemma}\label{lem:phi_or_psi_is_valid} If \( \phi \) and \( \psi \) are \( i \)-objective formulas, \( \phi, \psi \in \onlk \) and \( \know \phi \land \nknow \psi \) is  \( \axiomskplus \)-consistent then \( \models \phi \lor \psi \).

\end{lemma}

\begin{proof} 
	Suppose not. Then \( \neg \phi \land \neg \psi \) is \( \axiomsk \)-consistent, and by \( \afive^{t+1} \) we prove \( \naknow (\phi \lor \psi) \supset \neg \aknow (\phi \lor \psi) \), and thus, \( \naknow \psi \supset \neg \aknow \phi \), and this is not \( \axiomskplus \)-consistent with \( \aknow \phi \land \naknow \psi \). 
	\qed \end{proof}

\begin{theorem}\label{thm:completeness_onlk} For all \( \alpha \in \onlk \), if \( \models \alpha \) then \( \axiomsk \vdash \alpha \). 

\end{theorem}


\begin{proof} 
	Proof by induction on \( t \). It is sufficient to show that if a formula \( \beta \in \onlkplus \) is \( \axiomskplus \)-consistent then it is satisfiable \emph{wrt.}~some model. We already have the proof for \( \onl^1 \) (see Theorem \ref{thm:completeness_onlmin}). Let us assume the proof holds for all formulas \( \alpha \in \onlk \). Particularly, this means that any formula that is \( \axiomsk \)-consistent is satisfiable \emph{wrt.}~some \( (k',j') \)-model. Let \( \alpha \in \onlkplus \) (say of maximal \( a,b \)-depth of \( k+1,j+1 \)), and suppose that \( \alpha \) is consistent \emph{wrt.}~\( \axiomskplus \). It is sufficient to show that \( \alpha \) is satisfiable. Wlog, we take it in the normal form: \\[1ex]
		  \no \( \bigvee(\sigma \land \aknow \bel \land \neg \aknow \nbelone  \ldots \land \neg \aknow \varphi_{a{m_1}}  \land \bknow 	\varphi_{b0} \ldots \land \neg \bknow \varphi_{b{m_2}}  \land 
	 \naknow \most \ldots \land \neg \naknow \psi_{a{n_1}}   \land \nbknow \psi_{b0} \ldots \land \neg \nbknow \psi_{bn_2} ). \) \\[1ex]
\no Note that, by definition, it must be that all of \( \varphi_{im}, \psi_{in} \)~are at most in \( \onlk \) (i.e.~they may also be in \( \onl^{t-1}, \ldots
\)), and \( i \)-objective. We proceed as we did for Theorem \ref{thm:completeness_onlmin} but without restricting to basic formulas. Let \( A
\) be all \( \axiomsk \)-consistent formulas of the form \( \bel \land
\most \land \neg \nbel \) or \( \bel \land \most \land \neg \nmost \) (they are of maximal \( b \)-depth \( k \)). Let \(
\gamma \) be independent of all formulas in \( A \). Let \( S_a \) be the set of all (\(
\axiomsk \)-) maximally consistent sets of formulas, constructed from formulas of maximal \( b \)-depth \( k \), and containing \( \bel \land (\neg
\most \lor (\most \land \gamma)) \), and hence by induction hypothesis they are satisfiable in some model. Note that all formulas in \( S_a \) are
in \( \onlk \). The \( b \)-depth is maximally \( 2k+2 \). Letting \( k''=2k+2 \),
we have that for all \( S' \in S_a \), there is a \( \lan w, \ebkdashdash \ran
\) such that \( \{\}, \ebkdashdash, w \models S' \). Let \( k'=k''+1 \).
Letting \( {\eakdash} \) be all such \( k' \)-structures \( \lan w, \ebkdashdash \ran \) for each \( S' \in S_a \) makes \(
\aObj({\eakdash}) = S_a \) (in contrast, for Thereom \ref{thm:completeness_onlmin} we dealt with \( \canaObj \)). We claim that this \( k' \)-structure for
Alice, a \( j' \)-structure for Bob constructed similarly, and a world where
\( \sigma \) holds (there is such a world since \( \sigma \) is propositional
and consistent) is a model where \( \alpha \) is satisfied. The proof proceeds
as in Theorem \ref{thm:completeness_onlmin}. We show the case of \(
\neg \aknow \nbel \).

Since \( \aknow \bel \land \neg \aknow \nbel \) is consistent \emph{wrt.}~\( \axiomskplus
\), it must be that \( \bel \land \neg \nbel \) is consistent \emph{wrt.}~\(
\axiomskplus \). Further, since \( \bel, \nbel \in \onlk \), they must consistent be \emph{wrt.}~\( \axiomsk \) (for if not, they cannot by definition be consistent \emph{wrt.}~\( \axiomskplus \)). This means that either \( \bel \land \neg \nbel \land \most \)
or \( \bel \land \neg \nbel \land \neg \most \) is consistent. If the former
is, then so is \( \bel \land \neg \nbel \land \most \land \gamma \). Since \(
S_a \) consist of all \( \axiomsk \)-consistent formulas containing \(
\bel \land (\neg \most \land (\most \land \gamma)) \), there is clearly a \(
S' \in S_a \) such that \( \neg \nbel \in S' \). Consequently, it can not be
that \( {\eakdash}, \{\}, w' \models \aknow \nbel \). Thus, \( \eakdash, \{\},
w' \models \neg \aknow \nbel \).
%
\qed \end{proof}


%
%

\no Thus, we have a sound and complete axiomatization for the propositional
fragment of \( \onl \). In comparison to Lakemeyer (\citeyear{Lakemeyer1993}), the
axiomatization goes beyond a language that restricts the nesting of \( \nknow
\). In contrast to Halpern and Lakemeyer (\citeyear{1029713}), the axiomatization does not
necessitate the use of semantic notions in the proof theory. A third
axiomatization by \cite{DBLP:conf/aiml/Waaler04,DBLP:conf/tark/WaalerS05}
proposes an interesting alternative to deal with the circularity in a generalized \(
\mathbf{A5} \). The idea is to first define consistency by formulating a
fragment of the axiom system in the sequent calculus. Quite analogous to having
\( t \)-axioms, they allow us to apply \( \afive \) on \( i \)-objective
formulas of a lower depth, thus avoiding circularity without the need to
appeal to satisfiability as in \cite{1029713}.
Waaler and Solhaug~(\citeyear{DBLP:conf/tark/WaalerS05}) also define a
semantics for multi-agent only-knowing which does not appeal to
canonical models. Instead, they define a class of Kripke structures which need
to satisfy certain constraints. Unfortunately, these constraints are quite
involved and, as the authors admit, the nature of these models ``is complex
and hard to penetrate.'' 


To get a feel of the axiomatization, let us consider a well studied example from \cite{1029713} to see where we differ. Suppose Alice assumes the following default: unless I know that Bob knows my secret then he does not know it. If the default is all that she knows, then she \emph{nonmonotonically} comes to believe that Bob does not know her secret. Let \( 
\gamma \) be a proposition that denotes Alice's secret, and we want to show that \( \vdash \aoknow(\delta) \supset \aknow \neg \bknow \gamma \), where \( \delta = \neg \aknow \bknow \gamma \supset \neg \bknow \gamma \). We write (Def.)~to mean \( \aoknow \alpha \equiv \aknow \alpha   \land \naknow \neg \alpha \), and we freely reason with propositional logic (PL) or \( \kffn \).
 \begin{enumerate}
	\item \( \aoknow(\delta) \supset \aknow \neg \aknow \bknow \gamma \supset \aknow \neg \bknow \gamma  \) \hfill Def.,PL,\( \mathbf{A2}_n \) 

	\item \( \aoknow(\delta)\supset \naknow\neg \aknow \bknow \gamma \land \naknow\bknow \gamma \) \hfill Def.,PL,\( \kffn \) 

	\item \( \naknow \bknow \gamma \supset \neg \aknow \bknow \gamma \) \hfill \( \afive^1 \) 

	\item  \( \neg \aknow \bknow \gamma \supset \aknow \neg \aknow \bknow \gamma \) \hfill \( \mathbf{A4}_n \)    

	\item \( \aoknow(\delta) \supset \aknow \neg \aknow \bknow \gamma \) \hfill 2,3,4,PL        

	\item \( \aoknow(\delta) \supset \aknow \neg \bknow \gamma \) \hfill 1,5,PL
\end{enumerate}

\no We use \( \afive^1 \), and it is applicable because \( \neg \bknow \gamma \) is \( a \)-objective and \( \kffn \)-consistent. Now, suppose Alice is cautious. She changes her default to assume that if she does not believe Bob to only-know some set of facts \( \theta \in \Phi \), then \( \theta \) is not all that he knows. We would like to show
 \[ \vdash \aoknow(\neg \aknow \boknow \theta \supset \neg \boknow \theta) \supset \aknow \neg \boknow \theta \]         
Of course, this default is different from \( \delta \) in containing \( \boknow \theta \) rather than \( \bknow \gamma \). The proof is identical, except that we use \( \afive^2 \), since \( \neg \boknow \theta \in \onl^1 \) is \( a \)-objective and \( \axioms^1 \)-consistent. The latter proof requires reasoning with the \emph{satisfiability} modal operator in Halpern and Lakemeyer (\citeyear{1029713}), and is not provable with the axioms of Lakemeyer (\citeyear{Lakemeyer1993}).

\section{Autoepistemic Logic} 
\label{sec:autoepistemic_logic}



Having examined the properties of multi-agent only-knowing, in terms of a semantics for both the first-order and propositional case, and an axiomatization for the propositional case, in the current section we discuss how the semantics also captures autoepistemic logic (AEL). AEL, as originally developed by Moore (\citeyear{2781}), intends to allow agents to draw conclusions, by making observations of their own epistemic states. For instance, Alice concludes that she has no brother because if she did have one then she would have known about it, and she does not know about it \cite{2781}. The characterization of such beliefs are defined using fixpoints called \emph{stable expansions}. In the single agent case, Levesque (\citeyear{77758}) showed that the beliefs of an agent who only-knows \( \alpha \) is \emph{precisely} the stable expansion of \( \alpha \). Of course, the leverage with the former is that it is specified using regular entailments. In Lakemeyer (\citeyear{Lakemeyer1993}), and Halpern and Lakemeyer (\citeyear{1029713}), a many agent generalization of AEL is considered in the sense of a stable expansion for every agent, and relating this to what the agent only-knows. But their generalizations are only for the propositional fragment, while Levesque's definitions involved first-order entailments. In contrast, we obtain the corresponding quantificational multi-agent generalization of AEL. We state the main theorems below. The proofs are omitted since they follow very closely from the ideas for the single agent case~\cite{levesque2001logic}. 


\begin{definition} Let \( A \) be a set of formulas, and \( \Gamma \) is the \( i \)-stable expansion of \( A \) iff it the set of first-order implications of \( A \cup \{\know \beta \mid \beta \in \Gamma \} \cup \{\neg\know\beta \mid \beta \not\in\Gamma \} \). 

\end{definition}         

\newcommand{\eplusak}{\boldsymbol e_a^+}

\begin{definition}[Maximal structure] If \( \eak \) is a \( k \)-structure, let \( \eplusak \) be a \( k \)-structure with the addition of all \( \lan w', \ebkmin \ran \not \in \eak \) such that for every \( \alpha  \in \onlmin	\) of maximal \( a,b \)-depth \( k,k-1 \), if \( \eak, \{\}, w \models \aknow \alpha \) for any world \( w \) then \( \eak, \ebkmin, w' \models \alpha \). Define \( \Gamma = \{ \beta \mid \beta ~\textrm{is basic and}~ \eplusak, \{\}, w \models \aknow \beta \}  \) as the belief set of \( \eplusak \).

\end{definition}

\begin{theorem}\label{thm:onlyknowing_is_stable} Let \( M = \lan \eplusak, \ebj, w \ran \) be a model, where \( \eplusak \) is a maximal structure for \( a \). Let \( \Gamma \) be the belief set of \( \eplusak \), and suppose \( \alpha \in \onlmin \) is of maximal \( a,b \)-depth \( k,k-1 \). Then, 
	 \( M \models \aoknow\alpha \) iff \( \Gamma \) is the \( a \)-stable expansion of \( \alpha \).

\end{theorem}
                                               
\no Theorem \ref{thm:onlyknowing_is_stable} essentially says that the complete set of basic beliefs at a \emph{maximal} epistemic state where \( \alpha \) is all that \( i \) knows, precisely coincides with the \( i \)-stable expansion of \( \alpha \). 


\section{Axiomatizing Validity} 
\label{sec:axiomatizing_validity}

Extending the work in \cite{Lakemeyer1993} and \cite{DBLP:conf/aaai/Halpern93}, which was only restricted to formulas in \( \onlmin \), Halpern and Lakemeyer~(\citeyear{1029713}) proposed a multi-agent only-knowing logic that handles the nesting of \( \nknow \) operators. But as discussed, there are two undesirable features. The first is a semantics based on canonical models, and the second is a proof theory that axiomatizes validity. Although such a construction is far from natural,  we show in this section that they do indeed capture the desired properties of only-knowing. This also instructs us that our axiomatization avoids such problems in a reasonable manner.   
%
%
%

Recall that the language of \cite{1029713} is \( \onlp \), which is \( \onl \) and a modal operator for validity, \( \val \). A modal operator \( \sat \), for satisfiability, is used freely such that \( \val(\alpha) \) is syntactically equivalent to \( \neg \sat(\neg \alpha) \). To enable comparisons, we present a variant of our logic, that has all its main features, but has additional notions to handle the extended language. We then show that this logic and \cite{1029713} agree on the set of valid sentences from \( \onlp \) (and also \( \onl \)).

The main feature of \cite{1029713} is the proof theory \( \axioms' \), and a semantics that is sound and complete for \( \axioms' \) via the extended canonical model. \( \axioms' \) consists of \( \mathbf{A1}_n-\mathbf{A4}_n \), \( \mathbf{MP} \), \( \mathbf{NEC} \) and the following: 

\begin{itemize} 

\item[] \( \afive'. \) \( \sat(\neg \alpha) \supset (\nknow \alpha \supset \neg \aknow \alpha) \), if \( \alpha \) is \( i \)-objective. 

\item[]  \( \mathbf{V1}. \) \( \val (\alpha) \land \val(\alpha \supset \beta) \supset \val(\beta) \).  

\item[]  \( \mathbf{V2}. \) \( \sat(p_1 \land \ldots p_n) \), if \( p_i \)'s are literals and \( p_1 \land \ldots p_n \) is \\\qquad \mbox{}\qquad \mbox{}  propositionally consistent. 

\item[] \( \mathbf{V3}. \) \( \sat(\alpha \land \beta_1) \land \ldots \sat(\alpha \land \beta_k) \land \sat(\gamma \land \delta_1) \ldots \land \\\qquad \mbox{}\qquad \mbox{} \sat(\gamma \land \delta_m) \land \val(\alpha \lor \gamma) \supset \sat(\know \alpha \land \neg \know \neg \beta_1 \ldots \land \\\qquad \mbox{}\qquad \mbox{}\nknow \gamma \land  \neg \nknow \neg \delta_1 \ldots) \), if \( \alpha, \beta_i, \gamma, \delta_i \) are \( i \)-objective. 

\item[]  \( \mathbf{V4}. \) \( \sat(\alpha) \land \sat(\beta) \supset \sat(\alpha \land \beta) \), if \( \alpha \) is \( i \)-objective \\\qquad \mbox{}\qquad \mbox{}and \( \beta \) is \( i \)-subjective.

\item[]       \( \mathbf{NEC}_\val. \) From \( \alpha \) infer \( \val(\alpha) \).

\end{itemize}

%
%
%


\no The essence of our new logic, in terms of a notion of depth (with \( |\val(\alpha)|_i = |\alpha|_i \)) and a semantical account over possible worlds, is as before. The complete semantic definition for formulas in \( \onlp \) of maximal \( a,b \)-depth of \( k,j \) is: \begin{itemize}
	\item[1.] -8.~as before, 

	\item[9.] \( \eak, \ebj, w \models \val(\alpha) \) if \( \eak, \ebj, w\models \alpha \) for all \( \eak, \ebj,w \). 
\end{itemize}

\no {Satisfiability} and {validity} (\( \models \)) are understood analogously.\footnote{Note that \( \val \) corresponds precisely to how validity is defined.} Let \( \onlp ^1 \),  \( \ldots \) \( \onlp^t \) be also defined analogously. Further, let axioms \( \mathbf{A1}_n-\afive^{t} \) be defined for \( \onlp^{t} \). For instance, \( \afive^{t} \) is defined for any \( i \)-objective \( \neg \alpha \in \onlp^{t-1} \) that is consistent with \( \mathbf{A1}_n-\afive^{t-1} \). Then, the semantics above is characterized by the proof theory \( {\axioms^+}^t \) defined (inductively) for \( \onlp^t \), 
consisting of \( \axioms^t \) (\( \mathbf{A1}_n-\afive^t \), \( \mathbf{MP} \), \( \mathbf{NEC} \)) with \( \mathbf{NEC}_\val \) as an additional inference rule.

\begin{lemma}\label{lem:completeness_onlpk} For all \( \alpha \in \onlpk \), \( {\axioms^+}^t \vdash \alpha \) iff \( \models \alpha \). 

\end{lemma}

\no The proof of this lemma, and those of the following theorems are given in the appendix. 
We proceed to show that \( \sat(\alpha) \) is provable from \( \axioms' \) iff \( \alpha \) is \( {\axioms^+}^t \)-consistent. 
\begin{theorem}\label{thm:sat_means_consistent} For all \( \alpha \in \onlpk \), \( \axioms '\vdash \sat(\alpha) \) iff \( \alpha \) is \( {\axioms^+}^t \)-consistent. 

\end{theorem}

\no This allows us to show that \( \axioms' \)  and \( {\axioms^+}^t \) agree on provable sentences. 

\begin{theorem}\label{thm:provable_sentences} For all \( \alpha \in \onlp^t \), \( \axioms' \vdash \alpha \) iff \( {\axioms^+}^t \vdash \alpha \). 

\end{theorem}


\begin{lemma}\label{cor:final} For all \( \alpha \in \onlp ^t \), \( \models \alpha \) iff \( \alpha  \) is valid in\\~\cite{1029713}.
	
\end{lemma}

\begin{proof} 
	\( \axioms' \) is sound and complete for \cite{1029713}, and \( {\axioms^+}^t \) is sound and complete for \( \models \). 
	\qed \end{proof}

\no Since it can be shown that every \( \alpha \in \onlp \) is provably equivalent to some \( \alpha' \in \onl \) \cite{1029713}, we also obtain the following corollary. 

\begin{corollary}For all \( \alpha \in \onl ^t \), \( \models \alpha \) iff \( \alpha  \) is valid in ~\cite{1029713}.

\end{corollary}

%
%
%


\section{Conclusions} 
\label{sec:conclusions}

This paper has the following new results. We have a first-order modal logic for multi-agent only-knowing that we show, for the first time, generalizes Levesque's semantics. 
Unlike all attempts so far, we neither make use of proof-theoretic notions of maximal consistency nor Kripke structures \cite{DBLP:conf/tark/WaalerS05}. 
The benefit is that the semantic proofs are straightforward, and we understand possible worlds precisely as Levesque meant. We then analyzed a propositional subset, and showed first that the axiom system from Lakemeyer (\citeyear{Lakemeyer1993}) is sound and complete for a restricted language. We used this result to devise a new proof theory that does not require us axiomatize any semantic notions \cite{1029713}. 
Our axiomatization was shown to be sound and complete for the semantics, and its use is straightforward on formulas involving the nesting of \emph{at most} operators.
In the process, we revisited the features of only-knowing and compared the semantical framework to other approaches. Its behavior seems to coincide with our intuitions, and it also captures a multi-agent generalization of Moore's AEL. Finally, although the axiomatization of Halpern and Lakemeyer~(\citeyear{1029713}) is not natural, we showed that they essentially  capture the desired properties of multi-agent only-knowing, but at much expense.

\section{Acknowledgements} 
\label{sec:acknowledgements}

The authors would like to thank the reviewers for helpful suggestions and comments. The first author is supported by a DFG scholarship from the graduate school GK 643. 

                 
\bibliographystyle{jas99}
\bibliography{main}
         


\section{Appendix} 
\label{sec:appendix_2_}

	\no \textbf{Lemma~\ref{lem:completeness_onlpk}.} \emph{For all \( \alpha \in \onlpk \), \( {\axioms^+}^t \vdash \alpha \) iff \( \models \alpha \).}

\begin{proof} The proof is via induction. Using Theorems \ref{thm:soundness_full_lang} and \ref{thm:completeness_onlk} as the base cases in the induction, there is one additional step on the structure of formulas. 
	
	\emph{Soundness:}~The base case holds for formulas \( \alpha \in \onlk \) for \( \axiomsk \). Suppose now if \( \axiomsk \vdash \alpha \), then \( {\axioms^+}^t \vdash \val(\alpha) \). But if \( \axioms^t\vdash \alpha \) then (by induction hypothesis) at all models \(\eak, \ebj, w \models \alpha \), and so by the definition at all models \( \eak, \ebj, w\models \val(\alpha) \) or \( \models \val(\alpha) \). 
	
	\emph{Completeness.}~For the base case, we know that if for all models  \(\eak, \ebj, w \models \alpha \) then \( \axiomsk \vdash \alpha \). Suppose \( \models \val(\alpha) \), then by definition, for all models  \(\eak, \ebj, w  \models \alpha \) iff (by  hypothesis) \( \axiomsk \vdash \alpha \). So, \( {\axioms^+}^t \vdash \val(\alpha) \). 
\qed \end{proof}

\no \textbf{Theorem \ref{thm:sat_means_consistent}.}~\emph{For all \( \alpha \in \onlpk \), \( \axioms '\vdash \sat(\alpha) \) iff \( \alpha \) is \( {\axioms^+}^t \)-consistent. }


\begin{proof} 
It is helpful to have the following variant of Lemma \ref{lem:phi_or_psi_is_valid} at hand, and a corollary thereof. 
	
	\begin{lemma}\label{lem:val_sat_phi_or_psi} Suppose \( \phi, \psi \in \onlp ^{t-1} \) are \( i \)-objective \( {\axioms^+}^t \)-consistent formulas, and \( \models \phi \lor \psi \). Then \( \know \phi \land \nknow \psi \) is \( {\axioms^+}^t \)-consistent. 

	\end{lemma}

	\begin{proof} 
		Suppose not. Then \( {\axioms^+}^t\vdash \neg(\aknow \phi \land \naknow \psi) \), that is \( {\axioms^+}^t \vdash \neg \aknow \phi \lor \neg \naknow \psi \). Then, by Lemma \ref{lem:completeness_onlpk}, \( \models \neg \aknow \phi \lor \neg \naknow \psi \). Let \( \worlds_\phi = \{ w \mid w \models \phi \} \). Let \( \eak = \worlds_\phi \times \mathbb{E}^{k-1} \) be a \( e^k \) for Alice. Then clearly, \( \eak, \{\}, w\not\models \neg \aknow \phi \). It must be then that \( \eak, \{\}, w \models \neg \naknow \psi \). Then there is some \( \lan w', \ebkmin \ran \not\in\eak \) such that \( \eak, \ebkmin, w' \models \neg \psi \). And clearly, for all \( \lan w', \ebkmin \ran \not\in \eak \), \( \eak, \ebkmin, w' \models \neg \phi \) (by construction). It follows that there is a \( \lan w', \ebkmin \ran \not \in \eak \) where \( \eak,\\ \ebkmin, w' \models \neg (\phi \lor \psi) \), contradicting the validity of \( \phi \lor \psi \). \qed \end{proof}

	%

	\begin{corollary}\label{cor:sat_ln} Suppose \( \alpha, \beta_1, \ldots \beta_k, \gamma, \delta_1, \ldots \delta_m \in \onlp ^{t-1} \), are \( i \)-objective \( {\axioms^+}^t \)-consistent formulas, and \( \models \alpha \lor \gamma \). Then \( \know \alpha \land \neg \know \neg \beta_1  \ldots \land \neg \know \neg \beta_k \land \nknow \gamma \land \neg \nknow \delta_1  \ldots \land \neg \nknow \neg \delta_m \) is \( {\axioms^+}^t \)-consistent. \end{corollary}
	
\no	Returning to Theorem \ref{thm:sat_means_consistent}:~Proof on the  \emph{length} of the derivative, using induction on \( t \). Let \( \alpha \) be a consistent  propositional formula. Then, by \( \mathbf{V2} \), \( \axioms' \vdash \sat(\alpha) \). Since it is a consistent propositional formula, it is also \({\axioms^+}^t \)-consistent. Assume theorem holds for \( \alpha \in \onlp^{t-1} \). Suppose we have   \( \sat(\alpha \land \beta_k), \sat(\gamma \land \delta_m), \neg \sat(\neg (\alpha \lor \gamma)) \) \( \in \onlp^{t-1} \) then by \( \mathbf{V3} \), \( \axioms' \vdash \sat(\know \alpha \land \neg \know \neg \beta_k \land \nknow \gamma \land \neg \nknow \neg \delta_m) \). By  hypothesis \( \alpha \land \beta_k \), \( \gamma \land \delta_m \) are \( {\axioms^+}^t \)-consistent. And \( \neg (\alpha \lor \gamma) \) is not \( {\axioms^+}^t \)-consistent, and so \( {\axioms^+}^t \vdash \alpha \lor \gamma \). By Lemma \ref{lem:completeness_onlpk}, \( \models \alpha \lor \gamma \). Clearly, by Corollary \ref{cor:sat_ln}, \( \know \alpha \land \neg \know \neg \beta_k \land \nknow \gamma \land \neg \nknow \neg \delta_m \) is \( {\axioms^+}^t \)-consistent. Finally, suppose that you have \( \sat(\alpha) \) for some \( i \)-objective \( \alpha \) and \( \sat(\beta) \) for some \( i \)-subjective \( \beta \), then by \( \mathbf{V4} \), \( \axioms' \vdash \sat(\alpha \land \beta) \). By induction hypothesis, \( \alpha  \) and \( \beta \) are \( {\axioms^+}^t \)-consistent. By Lemma \ref{lem:completeness_onlpk}, \( \alpha \) is satisfiable and \( \beta \) is satisfiable, and so is \( \alpha \land \beta \). By Lemma \ref{lem:completeness_onlpk}, \( \alpha \land \beta \) is \( {\axioms^+}^t \)-consistent. The other direction is symmetric. 
	\qed \end{proof}

	
\no \textbf{Theorem \ref{thm:provable_sentences}.}~\emph{\( \axioms' \vdash \alpha \) iff \( {\axioms^+}^t \vdash \alpha \), for \( \alpha \in \onlp^t. \) }


\begin{proof} Since axioms \( \mathbf{A1}_n-\mathbf{A4}_n,\) $\mathbf{MP},$ $\mathbf{NEC},$ $\mathbf{NEC}_\val$ are
common to both, their use is not discussed. To show that \( \axioms' \vdash
\alpha \Rightarrow {\axioms^+}^t \vdash \alpha \), suppose you had \( \sat(\neg
\alpha) \) for some \( i \)-objective \( \alpha \in \onlp ^{t-1} \) then using \( \afive' \),
one could show that \( \nknow \alpha \supset \neg \know \alpha\). From Theorem \ref{thm:sat_means_consistent}, we also know \( \neg\alpha \) is
\( {\axioms^+}^{t-1} \)-consistent. Then, we can show \( \nknow
\alpha \supset \neg \know \alpha \) as well using \( \afive^t \). \( \mathbf{V2}, \mathbf{V3},
\mathbf{V4} \) follow immediately from Theorem \ref{thm:sat_means_consistent}.
Assuming now that the proof holds for base cases, using \textbf{V1}, if \( \axioms' \vdash \val(\alpha) \) and \( \axioms' \vdash \val(\alpha \supset \beta) \) then \( \axioms' \vdash \val(\beta) \). Now, by induction hypothesis, \( {\axioms^+}^t \vdash \val(\alpha) \) iff by Lemma \ref{lem:completeness_onlpk} \( \models \val(\alpha) \), and so \( \models \alpha \). Similarly, \( \models \alpha \supset \beta \), and thus, \( \models \beta \)  and \( \models 
\val(\beta) \) by the semantics. By Lemma~\ref{lem:completeness_onlpk}, \({\axioms^+}^t \vdash \val(\beta) \). 
	
	To show that \( {\axioms^+}^t \vdash \alpha \Rightarrow \axioms' \vdash \alpha \), suppose \( \neg \alpha \in \onlp ^{t-1} \) is \( i \)-objective and \( \axioms ^{t-1} \)-consistent, then one can prove \( \nknow 
\alpha \supset \neg \know \alpha \). Now,  \( \neg\alpha \) is also \( {\axioms^+}^t \)-consistent and by Theorem \ref{thm:sat_means_consistent}, \( \axioms' \vdash \sat(\neg\alpha) \). 
Then we can prove \( \nknow \alpha \supset \neg \know \alpha \), as desired.  
\qed \end{proof}

%
%


\end{document}